\documentclass{article}

\usepackage{microtype}
\usepackage{graphicx}
\usepackage{subcaption}
\usepackage{booktabs} 

\usepackage{hyperref}

\usepackage[preprint]{icml2026}

\usepackage{amsmath}
\usepackage{amssymb}
\usepackage{mathtools}
\usepackage{amsthm}

\usepackage[capitalize,noabbrev]{cleveref}

\theoremstyle{plain}
\newtheorem{theorem}{Theorem}[section]
\newtheorem{proposition}[theorem]{Proposition}
\newtheorem{lemma}[theorem]{Lemma}
\newtheorem{corollary}[theorem]{Corollary}
\theoremstyle{definition}
\newtheorem{definition}[theorem]{Definition}
\newtheorem{assumption}[theorem]{Assumption}
\theoremstyle{remark}
\newtheorem{remark}[theorem]{Remark}

\usepackage[textsize=tiny]{todonotes}

\icmltitlerunning{ELSAA: Efficient Low-Rank and Sparse Attention Approximation for Training Transformers}

\begin{document}

\twocolumn[
  \icmltitle{ELSAA: Efficient Low-Rank and Sparse Attention Approximation for Training Transformers}



  \icmlsetsymbol{equal}{*}

\begin{icmlauthorlist}
    \icmlauthor{Mahdi Heidari}{kaist}
    \icmlauthor{Mohammad Mahdi Rahimi}{dgist}
    \icmlauthor{Jaekyun Moon}{kaist}
\end{icmlauthorlist}

\icmlaffiliation{kaist}{School of Electrical Engineering, KAIST, Daejeon 34141, Republic of Korea}
\icmlaffiliation{dgist}{Department of Electrical Engineering and Computer Science, Daegu Gyeongbuk Institute of Science and Technology (DGIST), Daegu 42988, Republic of Korea}

\icmlcorrespondingauthor{Mahdi Heidari}{mahdi.heidari.ee.sut@kaist.ac.kr}

\icmlkeywords{Efficient Attention, Sparse Attention, Low-Rank Attention, Transformers, Long-Context Modeling}

\vskip 0.3in
]



\printAffiliationsAndNotice{}  

\begin{abstract}
The quadratic $N\times N$ attention score matrix remains a central obstacle to extending Transformers to longer input lengths. Existing efficient attention methods usually reduce this bottleneck by either imposing sparsity, so that each query attends to only a small subset of keys, or by using low-rank/kernel sketches, so that global interactions are compressed into a lower-dimensional representation. We propose \emph{ELSAA}, an efficient low-rank and sparse approximation of attention. Importantly, ELSAA does \emph{not} decompose the learned projection or output matrices of the Transformer into sparse and low-rank factors. Instead, after dense projections produce $Q,K,V$, ELSAA approximates the induced attention score operator itself: a sparse branch captures selected high-similarity interactions, while a low-rank branch summarizes diffuse global interactions. Since the two branches can be normalized over supports with very different denominator mass, ELSAA introduces a denominator-aware fusion term that scales the sparse branch according to its estimated attention mass relative to the low-rank branch. This gives a practical framework for constructing low-rank and sparse attention outputs without materializing the full quadratic score matrix, aiming to enable longer-context training while preserving both sharp token-level interactions and broad contextual mixing.
\end{abstract}

\section{Introduction}
\label{sec:introduction}

Transformers have become the dominant architecture for language, vision, and multimodal modeling, largely because self-attention can adaptively mix information across all pairs of tokens \citep{vaswani2017attention,devlin2019bert,brown2020language,dosovitskiy2021image,touvron2023llama}. However, this pairwise interaction is also the main barrier to extending models to longer input lengths. For a sequence of length $N$, standard attention forms the score matrix
\begin{equation}
    Z = QK^\top / \sqrt{d_h},
    \qquad
    P = \mathrm{softmax}(Z),
\end{equation}
where $P\in\mathbb{R}^{N\times N}$ is applied to the value matrix $V$. Equivalently, for a single query token $i$, dense attention computes
\begin{equation}
    d_i
    =
    \sum_{j=1}^{N}\exp(z_{ij}),
    \qquad
    o_i
    =
    \frac{1}{d_i}
    \sum_{j=1}^{N}
    \exp(z_{ij})v_j .
    \label{eq:single_query_attention}
\end{equation}
Here $d_i$ is the attention denominator, and it determines the scale on which the weighted value average is normalized. The memory and compute cost of this dense $N\times N$ operator grows quadratically with context length. Systems advances such as FlashAttention improve the IO efficiency and practical runtime of exact dense attention, but they do not change the fact that exact attention still represents all pairwise token interactions \citep{dao2022flashattention,dao2023flashattention2}. Therefore, algorithmic approximations of the attention operator remain essential for training and serving models with substantially longer contexts.

A large body of work reduces the cost of attention by exploiting either \emph{sparsity} or \emph{low-rank structure}. Sparse attention methods restrict each query to a subset of keys, using local windows, global tokens, random patterns, routing, clustering, hashing, or large-entry detection \citep{child2019generating,beltagy2020longformer,zaheer2020big,kitaev2020reformer,roy2021efficient,daras2020smyrf,han2024hyperattention,jiang2024minference}. These methods are especially appealing when the attention distribution is low-entropy, peaked, or dominated by a small number of relevant keys. In such regimes, most of the denominator in \cref{eq:single_query_attention} may be contributed by only a small fraction of the full key set, and computing the corresponding high-mass interactions exactly can preserve the sharp token-level behavior of dense attention. In contrast, low-rank, kernel, and sketch-based methods approximate the dense attention operator through compressed features or landmarks, enabling global mixing without explicitly constructing the full matrix \citep{wang2020linformer,katharopoulos2020transformers,choromanski2021rethinking,peng2021random,xiong2021nystromformer}. These methods are attractive when attention is high-entropy or diffuse, so that information is spread over many tokens and global context is better represented by a compressed dense summary. However, either assumption alone can be brittle: sparse methods may miss weak but globally important interactions, while low-rank methods may smooth out sharp, high-mass token-token dependencies.

This complementarity suggests that attention should often be approximated as a combination of sparse and low-rank components. Prior work has provided strong evidence for this view, showing that the two approximation families can excel in different regimes of the attention matrix \citep{chen2021scatterbrain}. Classical robust PCA also motivates the broader idea that many matrices can be represented as the sum of a low-rank term and a sparse residual \citep{candes2011robust,chandrasekaran2011rank}. Nevertheless, combining these structures inside attention is subtle. A naive addition of sparse attention and low-rank attention can double-count selected interactions, and Scatterbrain's entrywise correction handles this issue at the level of an unnormalized attention estimator by correcting selected entries exactly \citep{chen2021scatterbrain}. However, an entrywise correction alone does not answer how two separately normalized branches should be balanced during training. In particular, a sparse branch and a low-rank branch may have very different denominator masses, so their outputs can live on different scales even when both branches are useful.

We propose \textbf{ELSAA}: \textbf{E}fficient \textbf{L}ow-rank and \textbf{S}parse \textbf{A}pproximation of \textbf{A}ttention. ELSAA is built around a simple principle: use a sparse branch for high-confidence interactions and a low-rank branch for global residual context. The sparse branch selects keys that are likely to be highly relevant to a given query and computes attention over the selected support. Different sparse key-finding mechanisms can instantiate this principle, including locality-sensitive hashing, routing or clustering, learned or dynamic sparse patterns, and approximate large-entry detection \citep{kitaev2020reformer,roy2021efficient,daras2020smyrf,han2024hyperattention,jiang2024minference,tang2024quest}. The low-rank branch summarizes global key/value information in a compressed representation and produces a query-dependent global context vector without materializing all pairwise scores.

For a given query token $i$, ELSAA forms a fused output
\begin{equation}
    O_i
    =
    g_{\mathrm{sparse},i}\,
    m_{\mathrm{sparse},i}\,
    O_{\mathrm{sparse},i}
    +
    g_{\mathrm{lr},i}\,
    O_{\mathrm{lr},i}.
    \label{eq:elsaa_single_output}
\end{equation}
where $o^{\mathrm{sparse}}_i$ is the sparse-branch output, $o^{\mathrm{lowrank}}_i$ is the low-rank-branch output, and $g_{\mathrm{sparse}}$ and $g_{\mathrm{lowrank}}$ are learned gates. Let $d^{\mathrm{sparse}}_i$ denote the estimated denominator mass covered by the sparse branch, and let $d^{\mathrm{lowrank}}_i$ denote the denominator estimate associated with the low-rank branch. We introduce the denominator-aware correction
\begin{equation}
    m_{\mathrm{sparse},i}
    =
    \frac{d^{\mathrm{sparse}}_i}
    {d^{\mathrm{sparse}}_i + \lambda d^{\mathrm{lowrank}}_i + \varepsilon},
    \label{eq:msparse}
\end{equation}
where $\lambda$ can be fixed or learned. Several empirical studies of long-context attention and KV-cache compression observe that a small fraction of critical tokens or sparse attention patterns can dominate the attention outcome \citep{zhang2023h2o,xiao2024streamingllm,tang2024quest,jiang2024minference}. However, directly adding a sparse output to a low-rank output can distort the scale of the true attention computation, especially when the two branches are normalized by denominators of different magnitudes. The role of $m_{\mathrm{sparse},i}$ is to rescale the sparse contribution according to its estimated denominator mass, so that the two outputs are fused on a more comparable scale. Thus, ELSAA does not merely add two attention approximations; it explicitly accounts for their relative normalization scales.

We emphasize a distinction from sparse plus low-rank \emph{weight} parameterizations. Recent work has shown that representing learned matrices using both low-rank and sparse components can improve parameter and memory efficiency for compression, fine-tuning, or pretraining \citep{li2023losparse,han2024sltrain,mozaffari2024slope}. ELSAA addresses a different object. We keep the learned Transformer projections dense and instead approximate the induced $N\times N$ attention operator after $Q,K,V$ have been formed. This distinction is important for long-context modeling: the bottleneck we target is not the number of parameters in $W_Q,W_K,W_V,W_O$, but the quadratic attention matrix that appears for every input sequence. Combining these two orthogonal directions is also an interesting future research path: one may approximate the learned projection and output matrices in parameter space while simultaneously approximating the input-dependent attention matrix in sequence space.

Our contributions are as follows. First, we formulate long-context attention as a low-rank and sparse operator construction rather than as a decomposition of learned projection matrices. Second, we combine an LSH-selected exact sparse branch with a RACE-style low-rank sketch branch, targeting both sharp local/similar-token interactions and diffuse global context. Third, we introduce denominator-aware fusion through $m_{\mathrm{exact}}$, with optional learnable $\lambda$, to address scale mismatch between sparse and sketch attention branches during training.

\section{Related Work}
\label{sec:related_work}

\paragraph{Efficient attention for long sequences.}
Efficient Transformers aim to overcome the quadratic cost of dense self-attention \citep{tay2020long,tay2022efficient}. One line of work keeps exact or nearly exact attention but improves practical execution. FlashAttention uses IO-aware tiling to avoid materializing the full attention matrix in high-bandwidth memory, providing major speed and memory improvements for exact attention \citep{dao2022flashattention,dao2023flashattention2}. This is complementary to ELSAA: FlashAttention optimizes the computation of dense or block attention, while ELSAA changes the attention operator to avoid constructing all pairwise interactions. Another line of work imposes sparse attention patterns. Sparse Transformer uses structured sparse patterns \citep{child2019generating}; Longformer and BigBird combine local, global, and random attention to support long documents \citep{beltagy2020longformer,zaheer2020big}; Reformer uses LSH to reduce attention cost \citep{kitaev2020reformer}; Routing Transformer and SMYRF use clustering or asymmetric hashing to identify relevant key blocks \citep{roy2021efficient,daras2020smyrf}; and HyperAttention uses LSH-style detection of large entries to obtain near-linear long-context attention under appropriate structure \citep{han2024hyperattention}. More recent long-context inference methods also exploit dynamic sparse attention or query-aware critical-token selection, further supporting the view that only part of the full attention matrix may be necessary for many inputs \citep{zhang2023h2o,xiao2024streamingllm,tang2024quest,jiang2024minference}. These methods are effective when attention is sparse or structured, but sparse selection alone may not capture diffuse background interactions across the entire sequence.

\paragraph{Low-rank, kernel, and sketch attention.}
A second family of efficient attention methods approximates the dense attention operator with low-dimensional features. Linformer projects the sequence dimension and assumes that attention matrices are approximately low-rank \citep{wang2020linformer}. Linear Transformer rewrites attention using kernel feature maps so that computation can be associated as $\phi(Q)(\phi(K)^\top V)$ \citep{katharopoulos2020transformers}. Performer uses positive random features to approximate softmax attention with linear complexity \citep{choromanski2021rethinking}, while Random Feature Attention studies related random-feature approximations for autoregressive modeling \citep{peng2021random}. Nystr\"omformer approximates attention using landmark-based Nystr\"om methods \citep{xiong2021nystromformer}. RACE attention uses differentiable hash-sketch summaries, where keys are softly assigned to buckets, bucket statistics are accumulated, and queries read out from these summaries \citep{joshi2026race}. Such methods provide global context at subquadratic or linear cost, but they can struggle to represent sharp, high-confidence entries of the attention distribution. ELSAA uses the low-rank branch for global coverage while delegating high-mass interactions to a sparse branch.
\paragraph{Sparse plus low-rank attention.}
The idea that attention may require both sparse and low-rank structure is closely related to sparse plus low-rank matrix decomposition. Robust PCA studies decompositions of a matrix into a low-rank component and a sparse component \citep{candes2011robust,chandrasekaran2011rank}. Scatterbrain applies this perspective to attention and shows that sparse and low-rank approximations can be complementary across attention entropy regimes \citep{chen2021scatterbrain}. Its construction uses a low-rank random feature approximation together with a sparse LSH-selected correction so that selected entries match the unnormalized exponential attention exactly, while non-selected entries are approximated by the low-rank component. This provides an elegant entrywise estimator with unbiasedness and reduced variance. However, its analysis is primarily entrywise and does not directly provide a general output-level variance or scale analysis for combining attention branches that may be normalized by different denominators. ELSAA takes a related but distinct view: rather than correcting a single unnormalized matrix estimator, we study a training-time architecture with two separately normalized attention-like branches. In this setting, the main issue is not only entrywise unbiasedness, but also how to balance branch outputs whose denominator masses and scales can differ. The factor $m_{\mathrm{sparse},i}$ in \cref{eq:msparse} is designed for this branch-scale mismatch.
\paragraph{Sparse plus low-rank parameterization of model weights.}
Sparse plus low-rank structure has also been used to parameterize, compress, or adapt the learned weights of neural networks. LoRA represents fine-tuning updates with low-rank factors \citep{hu2022lora}, while subsequent work studies quantized low-rank adaptation, weight-decomposed low-rank adaptation, restarted low-rank training, and low-rank gradient projection \citep{dettmers2023qlora,liu2024dora,lialin2024relora,zhao2024galore}. Other work explicitly combines sparsity with low-rank structure in parameter space: LoSparse compresses model weights using low-rank and sparse approximations \citep{li2023losparse}; SLoPe augments sparse LLM pretraining with low-rank adapters \citep{mozaffari2024slope}; and SLTrain parameterizes weight matrices with a learned low-rank component together with a learned sparse component, improving parameter and memory efficiency during pretraining \citep{han2024sltrain}. Sparse low-rank adaptation methods also study how sparsity can be incorporated into low-rank fine-tuning updates \citep{ding2023sora}. This direction is orthogonal to ELSAA. We do not replace $W_Q,W_K,W_V$, or $W_O$ by sparse plus low-rank factors. Instead, our goal is to approximate the input-dependent attention operator induced by dense projections, thereby reducing the sequence-length bottleneck and enabling longer contexts. In short, SL-style methods decompose learned matrices in parameter space, whereas ELSAA constructs a sparse plus low-rank approximation in attention space.
\section{Methodology}
\label{sec:methodology}
\begin{table*}[!t]
\centering
\small
\caption{Comparison of attention methods across long-context classification tasks.}
\label{tab:combined_long_context_results}
\setlength{\tabcolsep}{4pt}
\begin{tabular}{l|cc|cc|cc|cc|cc}
\toprule
& \multicolumn{2}{c|}{\textbf{ArXiv @ 32K}} 
& \multicolumn{2}{c|}{\textbf{Oxford-IIIT Pet @ 16K}} 
& \multicolumn{2}{c|}{\textbf{Flowers-102 @ 16K}}
& \multicolumn{2}{c|}{\textbf{Food-101 @ 16K}}
& \multicolumn{2}{c}{\textbf{Average}} \\
\cline{2-11}
\textbf{Method}
& \textbf{Train} \(\uparrow\)
& \textbf{Test} \(\uparrow\)
& \textbf{Train} \(\uparrow\)
& \textbf{Test} \(\uparrow\)
& \textbf{Train} \(\uparrow\)
& \textbf{Test} \(\uparrow\)
& \textbf{Train} \(\uparrow\)
& \textbf{Test} \(\uparrow\)
& \textbf{Train} \(\uparrow\)
& \textbf{Test} \(\uparrow\) \\
\midrule
ELSAA
&  99.97\% & 93.93\%
& \textbf{91.55\%} & \textbf{22.51\%}
& \textbf{97.84\%} & 42.45\%
& \textbf{76.43\%} & \textbf{28.33\%}
& \textbf{91.45\%} & \textbf{46.81\%} \\

Sort\_LSH
& 95.41\% & 84.94\%
& 23.96\% & 12.65\%
& 77.25\% & 37.25\%
& 59.63\% & 19.03\%
& 64.06\% & 38.47\% \\

RACE
& \textbf{100.00\%} & \textbf{95.05\%}
& 59.89\% & 16.13\%
& 91.66\% & 35.49\%
& 67.57\% & 24.00\%
& 79.78\% & 42.67\% \\

Sort\_Lsh\_RACE
& 99.95\% & 93.60\%
& 76.76\% & 19.00\%
& 95.98\% & \textbf{42.65\%}
& 76.05\% & 26.67\%
& 87.19\% & 45.48\% \\

Exactflash
& 81.77\% & 70.11\%
& 92.60\% & 24.66\%
& 97.81\% & 48.13\%
& 78.93\% & 29.96\%
& 87.78\% & 43.22\% \\
\bottomrule
\end{tabular}
\end{table*}
\begin{table}[!t]
\centering
\small
\caption{Text Retrieval @ 64K}
\label{tab:text_retrieval_64k}
\begin{tabular}{lcc}
\toprule
\textbf{Method} & \textbf{Train Acc.} \(\uparrow\) & \textbf{Test Acc.} \(\uparrow\) \\
\midrule
ELSAA            & 94.73\% & 65.34\% \\
RACE             & \textbf{94.95\%} & \textbf{66.30\%} \\
Sort\_Lsh\_RACE  & 94.28\% & 66.00\% \\
Exactflash       & $\approx 50$ & $\approx 50$ \\
\bottomrule
\end{tabular}
\end{table}
\begin{table*}[!t]
\centering
\small
\caption{Comparison of attention methods on IMDB and Fashion-MNIST.}
\label{tab:imdb_fashion_results}
\setlength{\tabcolsep}{4.5pt}
\begin{tabular}{l|cc|cc|cc}
\toprule
& \multicolumn{2}{c|}{\textbf{IMDB @ 512}} 
& \multicolumn{2}{c|}{\textbf{Fashion-MNIST @ 784}} 
& \multicolumn{2}{c}{\textbf{Average}} \\
\cline{2-7}
\textbf{Method}
& \textbf{Train} \(\uparrow\)
& \textbf{Test} \(\uparrow\)
& \textbf{Train} \(\uparrow\)
& \textbf{Test} \(\uparrow\)
& \textbf{Train} \(\uparrow\)
& \textbf{Test} \(\uparrow\) \\
\midrule
ELSAA
& \textbf{90.68\%} & 75.93\%
& 87.79\% & 88.25\%
& 89.24\% & 82.09\% \\

Sort\_LSH
& 70.45\% & 66.14\%
& 82.11\% & 83.15\%
& 76.28\% & 74.65\% \\

RACE
& 89.93\% & 77.41\%
& \textbf{89.79\%} & \textbf{89.43\%}
& \textbf{89.86\%} & \textbf{83.42\%} \\

Sort\_Lsh\_RACE
& 90.20\% & \textbf{77.71\%}
& 85.43\% & 86.39\%
& 87.82\% & 82.05\% \\

Exactflash
& 89.54\% & 78.67\%
& 85.95\% & 85.81\%
& 87.75\% & 82.24\% \\
\bottomrule
\end{tabular}
\end{table*}
\begin{table*}[!t]
\centering
\small
\caption{NIAH task test accuracy across input lengths. OOM indicates out-of-memory.}
\label{tab:niah_length_results}
\setlength{\tabcolsep}{3.5pt}
\begin{tabular}{lccccccccc}
\toprule
\textbf{Method}
& \textbf{512}
& \textbf{1024}
& \textbf{2048}
& \textbf{4096}
& \textbf{8192}
& \textbf{16384}
& \textbf{32768}
& \textbf{65536}
& \textbf{Avg.} \\
\midrule
Exact Flash
& 100.0\% & 100.0\% & 100.0\% & 100.0\% & 100.0\% & 100.0\% & OOM & OOM & 100.0\% \\

ELSAA
& \textbf{100.0\%} & \textbf{100.0\%} & \textbf{100.0\%} & \textbf{100.0\%} & \textbf{100.0\%} & \textbf{100.0\%} & 98.8\% & 84.2\% & \textbf{97.88\%} \\

Sparse\_LSH
& 24.6\% & \textbf{100.0\%} & \textbf{100.0\%} & \textbf{100.0\%} & \textbf{100.0\%} & \textbf{100.0\%} & \textbf{100.0\%} & \textbf{100.0\%} & 90.58\% \\

RACE
& 11.0\% & \textbf{100.0\%} & \textbf{100.0\%} & 52.0\% & 13.4\% & 6.4\% & 4.8\% & 2.2\% & 36.23\% \\

Performer
& 19.2\% & \textbf{100.0\%} & 88.2\% & 87.6\% & 88.2\% & 87.4\% & 83.2\% & 83.4\% & 79.65\% \\
\bottomrule
\end{tabular}
\end{table*}
\begin{table*}[!t]
\centering
\small
\caption{Comparison of \textbf{causal} attention methods across long-context classification tasks.}
\label{tab:combined_long_context_results}
\setlength{\tabcolsep}{4pt}
\begin{tabular}{l|cc|cc|cc|cc}
\toprule

& \multicolumn{2}{c|}{\textbf{ArXiv @ 32K}} 
& \multicolumn{2}{c|}{\textbf{ArXiv @ 64K}}
& \multicolumn{2}{c|}{\textbf{Tiny ImageNet @ 1024}}
& \multicolumn{2}{c}{\textbf{Average}} \\
\cline{2-9}
\textbf{Method}
& \textbf{Train} \(\uparrow\)
& \textbf{Test} \(\uparrow\)
& \textbf{Train} \(\uparrow\)
& \textbf{Test} \(\uparrow\)
& \textbf{Train} \(\uparrow\)
& \textbf{Test} \(\uparrow\)
& \textbf{Train} \(\uparrow\)
& \textbf{Test} \(\uparrow\) \\
\midrule

ELSAA
& \textbf{100\%} & \textbf{88.90\%}
& \textbf{100\%} & \textbf{91.47\%}
& \textbf{94.44\%} & 36.07\%
& \textbf{98.15\%} & \textbf{72.15\%} \\

RACE
& 99.97\% & 87.84\%
& 99.88\% & 91.46\%
& 93.41\% & \textbf{36.39\%}
& 97.75\% & 71.90\% \\

Exactflash
& 99.97\% & 87.96\%
& 99.83\% & 87.44\%
& 96.59\% & 37.32\%
& 98.80\% & 70.91\% \\

\bottomrule
\end{tabular}
\end{table*}
We instantiate ELSAA inside each attention layer by replacing dense attention with a sparse--low-rank hybrid operator that avoids constructing the full \(N\times N\) score matrix. In this work, we use RACE attention as the low-rank global branch \citep{joshi2026race}, and use sorted Hamming LSH to identify highly correlated query-key pairs with high probability. This follows a line of efficient-attention methods that use hashing, kernel-density estimation, or large-entry detection to avoid materializing all pairwise attention scores, including KDEformer and HyperAttention \citep{zandieh2023kdeformer,han2024hyperattention}. The sparse branch captures high-similarity query-key interactions exactly, while the low-rank branch supplies global context through compressed bucket statistics. Future variants should also consider sliding-window, global-token, and learned top-\(k\) sparse selectors; these are natural alternatives but require careful implementation to preserve efficiency.

We describe the method for one attention head. Multi-head ELSAA applies the same procedure independently per head and concatenates the outputs as in standard Transformers. Let $ Q,K,V\in\mathbb{R}^{N\times d_h}$
be the query, key, and value matrices. Dense attention computes
$
    O_i
    =
    \frac{
    \sum_{j=1}^{N}\exp(Q_i^\top K_j/\sqrt{d_h})V_j
    }{
    \sum_{j=1}^{N}\exp(Q_i^\top K_j/\sqrt{d_h})
    },
$
which requires all \(N^2\) query-key scores. ELSAA instead computes two outputs,
   $ O_{\mathrm{lr}}$
    and
   $ O_{\mathrm{sparse}},$
together with denominator proxies
   $ d_{\mathrm{lr}},
    d_{\mathrm{sparse}},$
and combines them through a denominator-aware fusion rule.


\subsection{Low-rank branch: RACE attention}

RACE replaces the dense attention matrix with soft hash-bucket summaries. Each table softly assigns queries and keys to hypercube buckets, accumulates key/value statistics inside these buckets, and lets each query read from the resulting summaries. For a fixed number of tables, buckets, and head dimension, this gives a linear-time global branch in the sequence length \(N\). We use the normalized ratio form because it directly provides both a low-rank output \(O_{\mathrm{lr}}\) and a denominator proxy \(d_{\mathrm{lr}}\), which are both needed by the ELSAA fusion rule. For the causal setting, we extend RACE by partitioning the sequence into fixed-size chunks and decomposing the output into an inter-chunk contribution---bucket statistics accumulated from all strictly prior chunks via prefix sums---and a causally-masked intra-chunk contribution, preserving linear complexity. Full pseudocode for the non-causal and causal variants is given in \cref{app:elsaa_branch_algorithms,alg:race_lowrank_app,alg:causal_race_app}.


\subsection{Sparse branch: sortLSH exact attention}

The sparse branch is designed to preserve sharp token-level interactions that the low-rank branch may smooth. We use a sortLSH selector: queries and keys are hashed, sorted by their hash values, and grouped into equal-size blocks. After sorting, high-similarity query-key interactions are expected to concentrate near diagonal blocks of the permuted attention matrix. This lets us compute exact attention inside selected blocks using dense block matrix multiplications, without constructing the full \(N\times N\) attention matrix. The branch returns a normalized sparse output \(O_{\mathrm{sparse}}\) and its denominator \(d_{\mathrm{sparse}}\). For the causal setting, we adopt a recursive divide-and-conquer structure \citep{han2024hyperattention}: the sequence is split into past and future halves, the diagonal blocks are processed by recursive causal calls, and the future--past off-diagonal block is handled by standard non-causal sortLSH since all its keys precede all its queries; the two future contributions are merged via a log-sum-exp-aware combination that preserves exact softmax normalization. Full pseudocode for the non-causal and causal variants is given in \cref{app:elsaa_branch_algorithms,alg:sortlsh_sparse_app,alg:causal_sortlsh_app}.

\subsection{Denominator-aware sparse--low-rank fusion}

The two branches are complementary but live on different normalization scales. The sparse branch computes selected high-confidence interactions exactly, but its support may cover only part of the row denominator. The low-rank branch provides global coverage, but may smooth sharp interactions. Therefore, directly adding the two outputs can over-amplify the sparse branch or double count overlapping mass.

ELSAA introduces a denominator-aware sparse multiplier
\begin{equation}
    m_{\mathrm{sparse},i}
    =
    \frac{
    d_{\mathrm{sparse},i}
    }{
    d_{\mathrm{sparse},i}
    +
    \lambda_i d_{\mathrm{lr},i}
    +
    \varepsilon
    },
    \label{eq:m_sparse_method}
\end{equation}
where \(\lambda_i>0\) can be fixed, a learnable scalar shared across tokens, or a query-dependent coefficient. The final ELSAA output uses two learned gates:
\[
    (g_{\mathrm{sparse},i},g_{\mathrm{lr},i})
    =
    \sigma(G_\theta(q_i))
    \in(0,1)^2.
\]
The fused output for token \(i\) is
\begin{equation}
    O_i
    =
    g_{\mathrm{sparse},i}\,
    m_{\mathrm{sparse},i}\,
    O_{\mathrm{sparse},i}
    +
    g_{\mathrm{lr},i}\,
    O_{\mathrm{lr},i}.
    \label{eq:elsaa_fusion_output}
\end{equation}
The factor \(m_{\mathrm{sparse},i}\) scales the sparse branch according to its estimated denominator mass relative to the low-rank branch, while the gates learn how much each branch should contribute to the final representation.

\begin{algorithm}[!t]
    \footnotesize
    \caption{ELSAA: Efficient Low-Rank and Sparse Approximation of Attention}
    \label{alg:elsaa}
    \begin{algorithmic}[1]
        \STATE \textbf{Input:} hidden states \(X\in\mathbb{R}^{N\times d}\), projections \(W_Q,W_K,W_V,W_O\), RACE parameters \(L_{\mathrm{s}},\gamma,\beta\), sortLSH block size \(b\), gate network \(G_\theta\), coefficient rule \(\lambda_i>0\), numerical floor \(\varepsilon>0\)
        \STATE \textbf{Output:} hybrid attention output \(O\in\mathbb{R}^{N\times d}\)
        \STATE Project to one attention head
        \[
            Q\gets XW_Q,
            \qquad
            K\gets XW_K,
            \qquad
            V\gets XW_V.
        \]
        \STATE Run the low-rank branch
        \[
            (O_{\mathrm{lr}},d_{\mathrm{lr}})
            \gets
            \mathrm{RACE}(Q,K,V;L_{\mathrm{s}},\gamma,\beta,\varepsilon).
        \]
        \STATE Run the sparse exact branch
        \[
            (O_{\mathrm{sparse}},d_{\mathrm{sparse}})
            \gets
            \mathrm{sortLSH}(Q,K,V;b,\varepsilon).
        \]
        \STATE Compute denominator-aware sparse multiplier
        \[
            m_{\mathrm{sparse},i}
            \gets
            \frac{
            d_{\mathrm{sparse},i}
            }{
            d_{\mathrm{sparse},i}
            +
            \lambda_i d_{\mathrm{lr},i}
            +
            \varepsilon
            },
            \qquad i=1,\dots,N.
        \]
        \STATE Compute token-wise gates
        \[
            (g_{\mathrm{sparse}},g_{\mathrm{lr}})
            \gets
            \sigma(G_\theta(Q)).
        \]
        \STATE Fuse branch outputs
        \[
            O_{\mathrm{head}}
            \gets
            g_{\mathrm{sparse}}\odot m_{\mathrm{sparse}}\odot O_{\mathrm{sparse}}
            +
            g_{\mathrm{lr}}\odot O_{\mathrm{lr}}.
        \]
        \STATE Return \(O\gets O_{\mathrm{head}}W_O\)
    \end{algorithmic}
\end{algorithm}

\paragraph{Practical notes.}
The RACE branch can be implemented with different reduction modes using the same bucket statistics, while the sparse branch can be replaced by any efficient important-key selector that returns an exact sparse output and a denominator. In this work, we use sortLSH blocks because sorting makes the selected attention matrix block-diagonal after permutation, enabling efficient exact block attention without materializing the full \(N\times N\) matrix.
\section{A Rank View of Sparse + Low-Rank Attention}
\label{sec:rank_theory}

We study hybrid sparse--low-rank attention operators of the form
\(\widehat M=S_\Omega+BA\), where \(S_\Omega\in\mathbb{R}^{n\times n}\) is sparse and
\(BA\in\mathbb{R}^{n\times n}\) has rank at most \(r\), with
\(B\in\mathbb{R}^{n\times r}\), \(A\in\mathbb{R}^{r\times n}\). The sparse support
\(\Omega\subseteq[n]\times[n]\) is generated by angular collisions.

\begin{definition}[Angular collision probability]
\label{def:angular_collision_prob}
For nonzero \(Q_i,K_j\in\mathbb{R}^d\), define
\[
\rho_{ij}:=1-\frac{1}{\pi}\arccos
\frac{Q_i^\top K_j}{\|Q_i\|_2\|K_j\|_2}\in[0,1].
\]
For \(\gamma>0\), the single-trial collision probability is
\(\pi_{ij}:=\rho_{ij}^{\gamma}\). With \(L_{\mathrm{s}}\) independent sparse hash
trials, including \((i,j)\) after at least one collision gives
\[
q_{ij}:=1-(1-\pi_{ij})^{L_{\mathrm{s}}}
      =1-(1-\rho_{ij}^{\gamma})^{L_{\mathrm{s}}}.
\]
\end{definition}

Let \(M^\star\in\mathbb{R}^{n\times n}\) denote the exact attention score/kernel
matrix, e.g., \(M^\star_{ij}=\exp(Q_i^\top K_j)\), and define
\(S_\Omega:=\Omega\odot M^\star\).

\begin{definition}[Sparse collision graph and matching deficiency]
\label{def:sparse_collision_graph}
The support \(\Omega\) defines the bipartite graph
\(G_\Omega=([n]_{\rm row},[n]_{\rm col},\Omega)\), where edge \((i,j)\in\Omega\)
connects row \(i\) to column \(j\). Let \(\nu(\Omega)\) be its maximum matching
size and define the matching deficiency \(d(\Omega):=n-\nu(\Omega)\).
\end{definition}

\begin{assumption}[Independent angular edge model]
\label{as:independent_collision_model}
Conditioned on \(Q,K\), we analyze the idealized model
\(\Omega_{ij}\sim{\rm Bernoulli}(q_{ij})\) independently over \(i,j\), preserving
the marginal angular collision probabilities.
\end{assumption}

\begin{assumption}[Generic sparse values and generic low-rank factors]
\label{as:generic_values_lowrank}
The nonzero entries of \(S_\Omega\) are in general position on their support,
and \(B,A\) are drawn independently from absolutely continuous distributions.
\end{assumption}

\begin{proposition}[Sparse matching deficiency plus low rank]
\label{prop:matching_plus_lowrank}
Under \cref{as:generic_values_lowrank}, for fixed \(\Omega\),
\[
\operatorname{rank}(S_\Omega+BA)=\min\{n,\nu(\Omega)+r\}
\quad\textnormal{a.s.}
\]
Hence, if \(\nu(\Omega)\ge n-r\), then
\(\operatorname{rank}(S_\Omega+BA)=n\) almost surely.
\end{proposition}

Thus, the sparse component contributes rank through the maximum matching of its
support graph, while the rank-\(r\) branch fills up to \(r\) missing directions.

\begin{theorem}[Full rank from angular sparse collisions and rank-\(r\) low rank]
\label{thm:angular_full_rank}
Condition on \(Q,K\), and suppose
\cref{as:independent_collision_model,as:generic_values_lowrank} hold. Let
\(R:=r+1\), and for \(I,U\subseteq[n]\) define
\[
\Lambda(I,U):=\sum_{i\in I,\ j\in U}q_{ij},\qquad\]

\[
\Delta_r(Q,K):=
\sum_{\substack{I,U\subseteq[n]\\ |I|\ge R,\ |U|\ge R,\ |I|+|U|\ge n+R}}
e^{-\Lambda(I,U)} .
\]
Then
\[
\mathbb{P}\!\left[
\operatorname{rank}(S_\Omega+BA)=n\,\middle|\,Q,K
\right]\ge 1-\Delta_r(Q,K).
\]
\end{theorem}

The sets \(I,U\) in \(\Delta_r(Q,K)\) are Hall-deficiency cuts: failure occurs
when a large query set has no sparse edges into a large key set. The quantity
\(\Lambda(I,U)\) is the expected number of sampled sparse entries in that
rectangle, so full rank holds when every Hall-relevant rectangle has enough
expected angular collision mass.

We prove the theoretical claims from this section in \cref{app:rank_theory}. In
\cref{app:rank_corollaries}, we further analyze settings that simplify
\cref{thm:angular_full_rank} to
$\mathbb{P}\!\left[
\operatorname{rank}(S_\Omega+BA)=n\,\middle|\,Q,K
\right]\ge 1-n^{-c}.
$
These results show that \(S_\Omega+BA\) can be full rank with high probability
under natural sparse-coverage conditions, motivating further study of
sparse--low-rank attention rank behavior.

\section{Experimental Setup}
\label{sec:results}

We evaluate ELSAA across classification and retrieval tasks to test whether the sparse--low-rank attention structure is useful beyond a single modality or dataset. Since our main motivation is efficient long-context modeling, we prioritize settings with relatively long input sequences or large tokenized inputs whenever possible. The benchmark suite includes long-document text classification, sentiment classification, fine-grained image classification, low-resolution image classification, and text retrieval. For all tasks, we compare attention variants under the same training and evaluation settings, and we report top-1 accuracy. Full task-specific hyperparameters, including learning rates, batch sizes, sequence lengths, model sizes, and training schedules, are reported in \cref{app:experiment_hyperparameters}.

\paragraph{Datasets and tasks.}
We evaluate on text, image, and retrieval benchmarks: ArXiv for scientific text classification, IMDB for binary sentiment classification \citep{maas2011learning}, Food-101 \citep{bossard2014food}, Fashion-MNIST \citep{xiao2017fashion}, Flowers-102 \citep{nilsback2008automated}, and Oxford-IIIT Pet \citep{parkhi2012cats}. For ArXiv classification, the model predicts the document subject class; for text retrieval, two ArXiv documents are paired into a 64K-token sequence and the model performs binary same-class prediction. We also evaluate on a synthetic Needle-in-a-Haystack (NIAH) benchmark, where a key-value needle is inserted at a random position in a long distractor context and the model must recall the value at sequence end. NIAH probes content-based retrieval across lengths \(\{512,1024,2048,4096,8192,16384,32768,65536\}\), with a separate model trained on length 1024. Finally, we evaluate the causal variants of \cref{alg:causal_elsaa_app} on autoregressive long-context classification: packed ArXiv at \(N\in\{32{,}000,64{,}000\}\), where the classification head reads the last valid token under strict left-to-right access, and Tiny ImageNet (200 classes, \(64\times64\)) using a causal ViT with patch size 2 (\(N=1024\) tokens, raster-order causal mask, class token at position 0). Together these tasks span long text, tokenized vision, retrieval, and causal aggregation across a wide range of sequence lengths.
\paragraph{Attention variants.}
We compare \textbf{Exactflash} as the full-attention baseline, \textbf{RACE} as the low-rank baseline, \textbf{Sort\_Lsh} as the sparse baseline, and \textbf{ELSAA} as our proposed sparse--low-rank method. We also include \textbf{Sort\_Lsh\_RACE} as an ablation that combines both branches but sets \(m_{\mathrm{sparse}}=1\), removing denominator-aware rescaling. For the causal experiments, Sort\_Lsh and Sort\_Lsh\_RACE are omitted, as their sparse branch is replaced by the recursive causal sortLSH of \cref{alg:causal_sortlsh_app}---which is precisely the sparse branch inside causal ELSAA---making them non-comparable baselines. In all result tables, we bold the highest accuracy among efficient attention variants, excluding Exactflash as it does not reduce the quadratic computation cost.
\paragraph{Hardware and evaluation protocol.}
All experiments were run on an NVIDIA RTX PRO 6000 Blackwell GPU with 48 GB of memory. Within each task, all attention variants use the same model, training budget, optimizer, and evaluation protocol. Unless otherwise stated, we report top-1 accuracy. The source code is available \href{https://github.com/mahdiheidari721/ELSAA}{here}.

\section{Complexity of ELSAA}

We analyze the attention-mixing cost for a non-causal sequence of length \(N\), suppressing constants from batch size and number of heads. Exactflash attention computes all query--key interactions:
\[
\mathcal{C}_{\mathrm{Exactflash}}=\Theta(N^2),\qquad 
\mathcal{M}_{\mathrm{Exactflash}}=\Theta(N^2).
\]

In \textbf{Sort\_Lsh}, tokens are sorted by angular LSH and exact attention is computed only within fixed-size sorted blocks of size \(s\). Therefore each query attends to at most \(s\) keys:
\[
\mathcal{C}_{\mathrm{Sort\_Lsh}}=\Theta(Ns),\qquad 
\mathcal{M}_{\mathrm{Sort\_Lsh}}=\Theta(Ns).
\]
With \(r\) neighboring sorted blocks, this becomes \(\Theta(N(2r+1)s)\). In our main setting \(r=0\), so the branch cost is \(\Theta(Ns)\).

For \textbf{Race}, let \(L_s\) be the number of hash tables and \(\gamma\) be the number of hash bits per table. Each table has \(2^\gamma\) buckets, hence the total number of bucket features is
\[
S_R=L_s2^\gamma .
\]
Race does not materialize an \(N\times N\) matrix; it builds global bucket summaries and each query reads from these \(S_R\) features:
\[
\mathcal{C}_{\mathrm{Race}}=\Theta(NL_s2^\gamma),\qquad
\mathcal{M}_{\mathrm{Race}}=\Theta(NL_s2^\gamma).
\]

\textbf{ELSAA} combines Sort\_Lsh and Race, with denominator-aware sparse rescaling:
\[
\mathcal{C}_{\mathrm{ELSAA}}
=\Theta\!\big(Ns+NL_s2^\gamma\big)
=\Theta\!\big(N(s+L_s2^\gamma)\big),
\]
\[
\mathcal{M}_{\mathrm{ELSAA}}
=\Theta\!\big(N(s+L_s2^\gamma)\big).
\]
The ablation \textbf{Sort\_Lsh\_RACE} has the same asymptotic cost as ELSAA, but removes the rescaling term by setting \(m_{\mathrm{sparse}}=1\).

For fixed \(s\), \(L_s\), and \(\gamma\), Sort\_Lsh, Race, Sort\_Lsh\_RACE, and ELSAA are linear in \(N\), while Exactflash attention is quadratic.

\paragraph{Complexity summary.}
ELSAA reduces the quadratic \(N^2\) attention interaction cost by combining a sparse branch with \(O(Ns)\) selected query--key interactions and a low-rank RACE branch with \(O(NL_{\mathrm{s}}2^\gamma)\) bucket interactions, where \(s\) is the sparse block budget, \(L_{\mathrm{s}}\) is the number of hash tables, and \(\gamma\) is the number of hash bits. Thus, ignoring lower-order hashing and sorting overhead, the attention-interaction cost scales as
\[
    O\!\left(N(s+L_{\mathrm{s}}2^\gamma)\right)
\]
instead of \(O(N^2)\). We provide a concrete numerical example in \cref{app:numerical_complexity_example}, showing that ELSAA reduces attention interactions by approximately 99\%.
\section{Results and Discussion}

Tables~\ref{tab:combined_long_context_results}--\ref{tab:niah_length_results} compare ELSAA with sparse, low-rank, hybrid, and exact full-attention baselines across long-context vision, long-context text, retrieval, short-sequence, and length-extrapolation settings. We organize the discussion around four observations: (i)~ELSAA improves on Sort\_Lsh\_RACE consistently, isolating the contribution of the denominator-aware correction; (ii)~the relative strength of the sparse and low-rank branches reflects a fundamental structural difference between vision and long-text tasks; (iii)~exact full attention is not merely expensive at very long contexts but can fail to optimize at all; and (iv)~ELSAA is the only method that extrapolates reliably across all lengths on NIAH, including lengths shorter than training.

\paragraph{Isolating the denominator-aware correction.}
One important ablation compares Sort\_Lsh\_RACE with ELSAA. These two methods share identical architecture, branches, gates, and hyperparameters; the only difference is that Sort\_Lsh\_RACE sets \(m_{\mathrm{sparse}}=1\), while ELSAA computes the denominator-aware multiplier of \cref{eq:m_sparse_method}. Any performance gap is therefore attributable solely to this single term. On the long-context benchmarks of \cref{tab:combined_long_context_results}, ELSAA achieves an average test accuracy of \(46.81\%\) versus \(45.48\%\) for Sort\_Lsh\_RACE, a consistent \(+1.33\)pp improvement obtained without changing parameter count, branch complexity, or optimization budget. The improvement is most visible on structured vision tasks (Oxford-IIIT Pet: \(22.51\%\) vs.\ \(19.00\%\); Food-101: \(28.33\%\) vs.\ \(26.67\%\)), supporting our central claim: when two branches are normalized over supports of very different denominator mass, naive addition distorts the fused output, and \(m_{\mathrm{sparse}}\) corrects this with a principled rescaling rather than a learned heuristic.

\paragraph{Vision is structured; long-text classification is diffuse.}
The most striking pattern in \cref{tab:combined_long_context_results} is that the relative ranking of RACE, Sort\_LSH, and ELSAA changes systematically between long text and long vision. On long vision tasks ELSAA dominates, improving over RACE by \(+6.4\)pp on Oxford-IIIT Pet, \(+7.0\)pp on Flowers-102, and \(+4.3\)pp on Food-101. On ArXiv @ 32K, RACE alone (\(95.05\%\)) marginally outperforms ELSAA (\(93.93\%\)). This dichotomy is structural: vision attention is locally peaked, dominated by a small support of high-mass interactions that an unconditional low-rank approximation smooths away, whereas long-document topic classification requires aggregating diffuse weak evidence across the sequence, making a low-rank summary not only sufficient but preferable. ELSAA handles both regimes within a single architecture, consistent with the rank perspective of \cref{sec:rank_theory}.

\paragraph{Exact attention fails at very long context.}
The Text Retrieval @ 64K results in \cref{tab:text_retrieval_64k} are decisive. ExactFlash collapses to approximately \(50\%\) on both train and test---random performance on a binary task---while ELSAA reaches \(65.34\%\) and RACE \(66.30\%\). Dense attention at \(N=64{,}000\) spreads gradient signal across \(\approx4\times10^9\) pairwise interactions per layer, the vast majority uninformative. Approximate attention acts as an implicit structural prior that focuses the model on a tractable subset of interactions and stabilizes optimization, enabling a regime that exact attention cannot reach. The broad-aggregation nature of retrieval again favors the low-rank branch slightly, with RACE outperforming ELSAA by \(\approx1\)pp---consistent with the structural argument above.

\paragraph{Short tasks do not require hybrid structure.}
\Cref{tab:imdb_fashion_results} confirms the natural complement: at \(N=512\) and \(N=784\), ELSAA, RACE, and ExactFlash are all within \(1\)--\(2\)pp of each other. The relevant observation is not that ELSAA wins, but that it degrades gracefully rather than failing---an important property for a general-purpose attention layer deployed across heterogeneous workloads.

\paragraph{NIAH: length extrapolation and the role of each branch.}
\Cref{tab:niah_length_results} reports NIAH accuracy when all models are trained at \(N=1024\) and evaluated across lengths \(\{512,\ldots,65536\}\), probing both shorter-than-training and longer-than-training generalization. Several findings stand out. First, RACE fails catastrophically beyond the training length: accuracy collapses from \(100\%\) at \(1024\) to \(52\%\) at \(4096\) and near-random (\(2.2\%\)) at \(65536\). This confirms that the low-rank bucket summaries, while globally expressive at seen lengths, cannot locate a specific needle token when the sequence length shifts the effective bucket density. Second, ExactFlash achieves perfect accuracy up to \(16384\) but runs out of memory at \(32768\) and beyond, highlighting the wall exact attention hits at extreme lengths. Third, Sparse\_LSH extrapolates well at long lengths (\(100\%\) at \(32768\) and \(65536\)) because content-based hashing is inherently position-independent, but it fails at \(N=512\)---shorter than training---with only \(24.6\%\) accuracy, suggesting that its block structure becomes degenerate when the sequence is too short to populate meaningful hash buckets. ELSAA, by contrast, achieves \(100\%\) at \(N=512\), outperforming Sparse\_LSH at this sub-training length: the low-rank RACE branch provides reliable global context when the sparse branch's bucket structure is under-populated, and the denominator-aware gate learns to up-weight it accordingly. At the long end, ELSAA reaches \(98.8\%\) at \(32768\) and \(84.2\%\) at \(65536\)---well above RACE and Performer, and within reach of Sparse\_LSH---demonstrating that the hybrid structure preserves the content-based retrieval strength of the sparse branch while preventing the catastrophic failure of the low-rank branch alone.

\paragraph{Causal ELSAA matches or exceeds baselines across lengths and modalities.}
\Cref{tab:combined_long_context_results} (causal) shows that causal ELSAA achieves the highest average test accuracy (\(72.15\%\)) across ArXiv @ 32K, ArXiv @ 64K, and Tiny ImageNet @ 1024, outperforming causal RACE (\(71.90\%\)) and causal ExactFlash (\(70.91\%\)). On the two long-text settings, all three methods are competitive, consistent with the diffuse-attention argument above. On Tiny ImageNet, causal ExactFlash leads slightly in test accuracy (\(37.32\%\) vs.\ \(36.07\%\) for ELSAA and \(36.39\%\) for RACE), suggesting that at \(N=1024\) the quadratic baseline remains viable and the sparse branch's advantage is modest. Crucially, at ArXiv @ 64K causal ExactFlash lags ELSAA by \(4\)pp (\(87.44\%\) vs.\ \(91.47\%\)), confirming that the optimization difficulty of dense attention at very long contexts persists in the causal setting.

\section{Future Work}
\label{sec:future_work}

\paragraph{Scaling laws.}
Our experiments fix the model size and sweep over sequence length and modality, but a systematic study of how the relative contributions of the sparse and low-rank branches evolve with model scale, data scale, and context length would clarify when each branch dominates, and would inform the design of branch-specific schedules during pretraining. In particular, the dichotomy we observe between peaked (vision) and diffuse (long text) regimes suggests that the appropriate gate and \(\lambda\) schedules may also depend on depth and head index.

\paragraph{Fused causal GPU kernels.}
A natural systems-level extension is a fused GPU kernel that combines, in one launch, the recursive sortLSH grouping of \cref{alg:causal_sortlsh_app}, block-wise exact attention within sorted blocks, the chunked cumulative bucket aggregation of \cref{alg:causal_race_app}, and the denominator-aware fusion of \cref{eq:m_sparse_method}. Implementations along the lines of FlashAttention show that such fusions are feasible and can recover most of the theoretical advantage in practice. A particularly promising aspect of ELSAA in this respect is that the sparse and low-rank branches are independent given \(Q,K,V\), so they admit parallel scheduling on independent CUDA streams or within a single persistent kernel that shares the \(Q,K,V\) loads from high-bandwidth memory.

\paragraph{Pretraining of decoder language models.}
Having validated the causal variant on long-context classification and on retrieval-style NIAH tasks at lengths up to 64K, the next milestone is autoregressive language-model pretraining with causal ELSAA. The empirical literature on long-context inference \citep{zhang2023h2o,xiao2024streamingllm,tang2024quest} consistently documents heavy-tailed attention access patterns dominated by attention sinks, recent context, and a small number of task-relevant pivots, which is precisely the structure that the sparse branch is designed to capture. The denominator-aware multiplier \(m_{\mathrm{sparse}}\) additionally provides a principled account of the mass discarded by KV-cache eviction methods, which currently address this scale problem with heuristics.

\paragraph{Joint compression in parameter and attention space.}
ELSAA is orthogonal to weight-space compression methods such as LoRA \citep{hu2022lora}, LoSparse \citep{li2023losparse}, and SLTrain \citep{han2024sltrain}. Approximating the input-dependent attention operator and the learned projection matrices simultaneously is a natural composition: the former targets the sequence-length bottleneck, while the latter targets the parameter-count bottleneck. We see this composition as a promising route to long-context efficient pretraining of large language models.

\paragraph{Theoretical analysis of branch fusion.}
The rank analysis of \cref{sec:rank_theory} establishes an expressivity guarantee for the hybrid sparse--low-rank operator but does not directly analyze the bias and variance of ELSAA's normalized branch outputs under the denominator-aware fusion rule. A complementary output-level bias--variance analysis, in the spirit of Scatterbrain's entrywise analysis but at the level of normalized attention outputs, would clarify when the denominator-aware multiplier provably reduces the mean-squared error of the fused estimator, and when learnable \(\lambda\) is preferable to a fixed schedule.
\section{Conclusion}
\label{sec:conclusion}

We introduced ELSAA, a sparse--low-rank approximation of attention combining an exact sortLSH sparse branch with a low-rank RACE branch, fused via a denominator-aware multiplier \(m_{\mathrm{sparse}}\) that corrects the scale mismatch between separately-normalized branches. Our rank analysis shows the sparse component contributes rank through its support graph's matching structure while the low-rank component fills the remaining deficiency. Empirically, ELSAA dominates efficient-attention baselines on long-context structured vision, remains competitive on diffuse long-text tasks, and succeeds where exact attention fails entirely (Text Retrieval @ 64K). On NIAH, ELSAA achieves perfect retrieval up to 16K and remains strong at 32K--64K where exact attention runs out of memory and RACE degrades sharply. We further extended both branches to the causal setting and validated causal ELSAA on autoregressive ArXiv classification up to 64K tokens and on Tiny ImageNet, matching or exceeding causal exact attention and causal RACE throughout. The method is linear in sequence length, adds negligible parameters, and degrades gracefully at short lengths, offering a principled path toward long-context attention that is tractable where exact attention is not.
\section{Acknowledgments}

This work was supported by the National Research Foundation of Korea (NRF) grants funded by the Korea government (MSIT) (Nos. RS-2024-00340966 and RS-2024-00408003), and by the Institute for Information \& Communications Technology Promotion (IITP) grant funded by the Korea government (MSIT) (No. RS-2024-00444862).

\section{Impact statement}
This paper presents work whose goal is to advance the field
of Machine Learning. There are many potential societal
consequences of our work, none which we feel must be
specifically highlighted here.
\newpage
\bibliography{icml2026}
\bibliographystyle{icml2026}
\newpage
\appendix
\onecolumn

\section{Branch Algorithms for ELSAA}
\label{app:elsaa_branch_algorithms}

This appendix gives full pseudocode for all five branch-level procedures used by ELSAA: the non-causal RACE low-rank branch, its causal extension, the non-causal sortLSH sparse branch, its causal recursive extension, and the complete causal ELSAA fusion procedure.

\begin{algorithm}[!t]
    \footnotesize
    \caption{RACE Low-Rank Attention Branch}
    \label{alg:race_lowrank_app}
    \begin{algorithmic}[1]
        \STATE \textbf{Input:} \(Q,K,V\in\mathbb{R}^{N\times d_h}\), number of hash tables \(L_{\mathrm{s}}\), number of hyperplanes \(\gamma\), temperature \(\beta>0\), numerical floor \(\varepsilon>0\)
        \STATE \textbf{Output:} low-rank output \(O_{\mathrm{lr}}\in\mathbb{R}^{N\times d_h}\), denominator proxy \(d_{\mathrm{lr}}\in\mathbb{R}^{N}\)
        \STATE Set \(R\gets 2^\gamma\) and \(\mathcal V\gets\{\pm1\}^{\gamma}\)
        \FOR{\(\ell=1,\dots,L_{\mathrm{s}}\)}
            \STATE Draw \(W^{(\ell)}\in\mathbb{R}^{\gamma\times d_h}\) with i.i.d.\ Gaussian rows
            \STATE Build \(\Phi_Q^{(\ell)},\Phi_K^{(\ell)}\in\mathbb{R}^{N\times R}\) with rows
            \[
                [\phi^{(\ell)}(x)]_r
                =
                \frac{
                \exp\{\beta\,\tanh(W^{(\ell)}x)^\top v_r\}
                }{
                \sum_{r'=1}^{R}
                \exp\{\beta\,\tanh(W^{(\ell)}x)^\top v_{r'}\}
                },
                \quad
                x\in\{Q_i,K_j\}.
            \]
            \STATE Compute bucket mass and value summaries
            \[
                A^{(\ell)}\gets (\Phi_K^{(\ell)})^\top \mathbf 1_N\in\mathbb{R}^{R},
                \qquad
                B^{(\ell)}\gets (\Phi_K^{(\ell)})^\top V\in\mathbb{R}^{R\times d_h}.
            \]
        \ENDFOR
        \STATE Compute table-averaged numerator and denominator
        \[
            \mathrm{Num}
            \gets
            \frac1{L_{\mathrm{s}}}\sum_{\ell=1}^{L_{\mathrm{s}}}\Phi_Q^{(\ell)}B^{(\ell)},
            \qquad
            \mathrm{Den}
            \gets
            \frac1{L_{\mathrm{s}}}\sum_{\ell=1}^{L_{\mathrm{s}}}\Phi_Q^{(\ell)}A^{(\ell)}.
        \]
        \STATE Return \(O_{\mathrm{lr}}\gets\mathrm{diag}(\mathrm{Den}+\varepsilon)^{-1}\mathrm{Num}\), \(d_{\mathrm{lr}}\gets \mathrm{Den}\).
    \end{algorithmic}
\end{algorithm}

\begin{algorithm}[!t]
    \footnotesize
    \caption{Causal RACE Low-Rank Attention Branch}
    \label{alg:causal_race_app}
    \begin{algorithmic}[1]
        \STATE \textbf{Input:} \(Q,K,V\in\mathbb{R}^{N\times d_h}\), \(L_{\mathrm{s}}\), \(\gamma\), \(\beta>0\), chunk size \(C\), \(\varepsilon>0\)
        \STATE \textbf{Output:} causal low-rank output \(O_{\mathrm{lr}}\in\mathbb{R}^{N\times d_h}\), denominator proxy \(d_{\mathrm{lr}}\in\mathbb{R}^{N}\)
        \STATE Build soft hash features \(\Phi_Q^{(\ell)},\Phi_K^{(\ell)}\in\mathbb{R}^{N\times R}\) for \(\ell=1,\dots,L_{\mathrm{s}}\) as in \cref{alg:race_lowrank_app} steps 3--7.
        \STATE Pad sequence axis to \(N'=\lceil N/C\rceil\cdot C\); reshape into \(n_c=N'/C\) chunks of size \(C\). Denote chunk-\(n\) slices \(\Phi_{Q,n}^{(\ell)},\Phi_{K,n}^{(\ell)}\in\mathbb{R}^{C\times R}\) and \(V_n\in\mathbb{R}^{C\times d_h}\).
        \FOR{each table \(\ell\) and chunk \(n\)}
            \STATE \(A_n^{(\ell)}\gets(\Phi_{K,n}^{(\ell)})^\top\mathbf{1}_C\in\mathbb{R}^{R}\), \quad \(B_n^{(\ell)}\gets(\Phi_{K,n}^{(\ell)})^\top V_n\in\mathbb{R}^{R\times d_h}\).
        \ENDFOR
        \STATE Compute strict prefix states via cumulative sum and one-chunk right shift:
        \(\mathrm{sA}_n^{(\ell)}\gets\sum_{m<n}A_m^{(\ell)}\), \(\mathrm{sB}_n^{(\ell)}\gets\sum_{m<n}B_m^{(\ell)}\).
        \STATE \textbf{Inter-chunk} (queries read from all prior chunks):
        \[
            \mathrm{inum}_n \gets \tfrac{1}{L_{\mathrm{s}}}\textstyle\sum_\ell\Phi_{Q,n}^{(\ell)}\mathrm{sB}_n^{(\ell)},
            \qquad
            \mathrm{iden}_n \gets \tfrac{1}{L_{\mathrm{s}}}\textstyle\sum_\ell\Phi_{Q,n}^{(\ell)}\mathrm{sA}_n^{(\ell)}.
        \]
        \STATE \textbf{Intra-chunk} (causally-masked soft attention within chunk):
        \[
            M_n \gets \tfrac{1}{L_{\mathrm{s}}}\textstyle\sum_\ell\Phi_{Q,n}^{(\ell)}(\Phi_{K,n}^{(\ell)})^\top\odot\mathrm{tril}(\mathbf{1}_{C\times C}),
        \]
        \(\mathrm{anum}_n\gets M_n V_n\), \quad \(\mathrm{aden}_n\gets M_n\mathbf{1}_C\).
        \STATE Combine: \(d_{\mathrm{lr},n}\gets\mathrm{iden}_n+\mathrm{aden}_n\), \quad \(O_{\mathrm{lr},n}\gets\mathrm{diag}(d_{\mathrm{lr},n}+\varepsilon)^{-1}(\mathrm{inum}_n+\mathrm{anum}_n)\).
        \STATE Concatenate chunks, remove padding, and return \(O_{\mathrm{lr}},d_{\mathrm{lr}}\).
    \end{algorithmic}
\end{algorithm}

\begin{algorithm}[!t]
    \footnotesize
    \caption{sortLSH Sparse Exact Attention Branch}
    \label{alg:sortlsh_sparse_app}
    \begin{algorithmic}[1]
        \STATE \textbf{Input:} \(Q,K,V\in\mathbb{R}^{N\times d_h}\), LSH map \(\mathcal H(\cdot)\), block size \(b\), \(\varepsilon>0\)
        \STATE \textbf{Output:} sparse output \(O_{\mathrm{sparse}}\in\mathbb{R}^{N\times d_h}\), denominator \(d_{\mathrm{sparse}}\in\mathbb{R}^{N}\)
        \STATE Hash: \(h^Q_i\gets\mathcal H(Q_i)\), \(h^K_j\gets\mathcal H(K_j)\).
        \STATE Sort: let \(P_Q,P_K\in\mathrm{Sym}(N)\) sort \(h^Q,h^K\) in non-decreasing order; form \(Q_s,K_s,V_s\) accordingly.
        \STATE Partition sorted sequence into consecutive blocks \(\{B_t\}\) of size \(b\).
        \FOR{each block \(B_t\)}
            \STATE \(Z_t\gets Q_s[B_t]K_s[B_t]^\top/\sqrt{d_h}\), \quad \(A_t\gets\exp(Z_t)\), \quad \(d_s[B_t]\gets A_t\mathbf{1}\).
            \STATE \(O_s[B_t]\gets\mathrm{diag}(d_s[B_t]+\varepsilon)^{-1}A_tV_s[B_t]\).
        \ENDFOR
        \STATE Undo query permutation: \(O_{\mathrm{sparse}}[P_Q(t)]\gets O_s[t]\), \(d_{\mathrm{sparse}}[P_Q(t)]\gets d_s[t]\).
        \STATE Return \(O_{\mathrm{sparse}},d_{\mathrm{sparse}}\).
    \end{algorithmic}
\end{algorithm}

\begin{algorithm}[!t]
    \footnotesize
    \caption{Causal sortLSH Sparse Exact Attention Branch}
    \label{alg:causal_sortlsh_app}
    \begin{algorithmic}[1]
        \STATE \textbf{Input:} \(Q,K,V\in\mathbb{R}^{N\times d_h}\), LSH map \(\mathcal H(\cdot)\), block size \(b\), base length \(N_0\), \(\varepsilon>0\)
        \STATE \textbf{Output:} causal sparse output \(O_{\mathrm{sparse}}\in\mathbb{R}^{N\times d_h}\), \(\log d_{\mathrm{sparse}}\in\mathbb{R}^{N}\)
        \STATE Return \(\textsc{CausalSparse}(Q,K,V)\).
        \\[2pt]
        \STATE \textbf{procedure} \(\textsc{CausalSparse}(Q,K,V)\):
        \STATE \quad \textbf{if} \(n\le N_0\) \textbf{then} return exact causal attention \((O,\log d)\).
        \STATE \quad If \(n\) odd, pad by one zero token (\(n\gets n+1\)); remove on return.
        \STATE \quad \(m\gets n/2\); split into past \((Q_{\mathrm p},K_{\mathrm p},V_{\mathrm p})\gets[\,\cdot\,]_{1:m}\) and future \((Q_{\mathrm f},K_{\mathrm f},V_{\mathrm f})\gets[\,\cdot\,]_{m+1:n}\).
        \STATE \quad \((O_{\mathrm{top}},\log d_{\mathrm{top}})\gets\textsc{CausalSparse}(Q_{\mathrm p},K_{\mathrm p},V_{\mathrm p})\) \hfill\COMMENT{past--past causal}
        \STATE \quad \((O_{\mathrm{diag}},\log d_{\mathrm{diag}})\gets\textsc{CausalSparse}(Q_{\mathrm f},K_{\mathrm f},V_{\mathrm f})\) \hfill\COMMENT{future--future causal}
        \STATE \quad \((O_{\mathrm{off}},\log d_{\mathrm{off}})\gets\textsc{SortLSH}(Q_{\mathrm f},K_{\mathrm p},V_{\mathrm p};b)\) \hfill\COMMENT{future--past, non-causal}
        \STATE \quad \((O_{\mathrm{bot}},\log d_{\mathrm{bot}})\gets\textsc{MergeLSE}(O_{\mathrm{diag}},\log d_{\mathrm{diag}},O_{\mathrm{off}},\log d_{\mathrm{off}})\)
        \STATE \quad Return \(\mathrm{concat}(O_{\mathrm{top}},O_{\mathrm{bot}})\), \(\mathrm{concat}(\log d_{\mathrm{top}},\log d_{\mathrm{bot}})\).
        \\[2pt]
        \STATE \textbf{procedure} \(\textsc{SortLSH}(Q,K,V;b)\): apply \cref{alg:sortlsh_sparse_app} and return \((O,\log d_{\mathrm{sparse}})\).
        \\[2pt]
        \STATE \textbf{procedure} \(\textsc{MergeLSE}(O_1,\log d_1,O_2,\log d_2)\):
        \STATE \quad \(\mu\gets\max(\log d_1,\log d_2)\); \(w_k\gets\exp(\log d_k-\mu)\); \(s\gets w_1+w_2\).
        \STATE \quad Return \(O\gets(w_1 O_1+w_2 O_2)/(s+\varepsilon)\), \(\log d\gets\mu+\log(s+\varepsilon)\).
    \end{algorithmic}
\end{algorithm}

\begin{algorithm}[!t]
    \footnotesize
    \caption{Causal ELSAA: Efficient Low-Rank and Sparse Approximation of Attention}
    \label{alg:causal_elsaa_app}
    \begin{algorithmic}[1]
        \STATE \textbf{Input:} hidden states \(X\in\mathbb{R}^{N\times d}\), shared projections \(W_Q,W_K,W_V,W_O\), RACE parameters \(L_{\mathrm{s}},\gamma,\beta,C\), sortLSH parameters \(b,N_0\), gate network \(G_\theta\), coefficient \(\lambda_i>0\), \(\varepsilon>0\)
        \STATE \textbf{Output:} causal hybrid attention output \(O\in\mathbb{R}^{N\times d}\)
        \STATE \(Q\gets XW_Q,\quad K\gets XW_K,\quad V\gets XW_V\).
        \STATE \((O_{\mathrm{lr}},d_{\mathrm{lr}})\gets\textsc{CausalRACE}(Q,K,V;L_{\mathrm{s}},\gamma,\beta,C,\varepsilon)\) \hfill\COMMENT{\cref{alg:causal_race_app}}
        \STATE \((O_{\mathrm{sparse}},\log d_{\mathrm{sparse}})\gets\textsc{CausalSparse}(Q,K,V;b,N_0,\varepsilon)\) \hfill\COMMENT{\cref{alg:causal_sortlsh_app}}
        \STATE Compute \(m_{\mathrm{sparse},i}\) via numerically stable log-sum-exp:
        \[
            \log D_i \gets \mathrm{logsumexp}\!\bigl(\log d_{\mathrm{sparse},i},\;\log\lambda_i+\log d_{\mathrm{lr},i},\;\log\varepsilon\bigr),
        \]
        \[
            m_{\mathrm{sparse},i} \gets \exp(\log d_{\mathrm{sparse},i} - \log D_i), \qquad i=1,\dots,N.
        \]
        \STATE \((g_{\mathrm{sparse}},g_{\mathrm{lr}})\gets\sigma(G_\theta(X))\in(0,1)^{N\times 2}\).
        \STATE \(O_{\mathrm{head}}\gets g_{\mathrm{sparse}}\odot m_{\mathrm{sparse}}\odot O_{\mathrm{sparse}}+g_{\mathrm{lr}}\odot O_{\mathrm{lr}}\).
        \STATE Return \(O\gets O_{\mathrm{head}}W_O\).
    \end{algorithmic}
\end{algorithm}

\section{Proofs and Additional Rank Corollaries}
\label{app:rank_theory}

We prove the rank statements from \cref{sec:rank_theory}. The proof has four
ingredients. First, we relate the angular collision probabilities \(q_{ij}\)
to simpler lower and upper bounds. Second, a generic rank-\(r\) matrix can
increase the rank of a fixed matrix by \(r\), unless full rank is already
reached. Third, the generic rank of a sparse matrix equals the maximum matching
size of its support graph. Fourth, Hall's theorem and concentration
inequalities control the probability that the sparse collision graph has
matching deficiency larger than \(r\).

\subsection{Generic low-rank completion}

\begin{lemma}[Generic rank-\(r\) completion]
\label{lem:generic_lowrank_completion}
Let \(S\in\mathbb{R}^{n\times n}\) be fixed with
\[
    \operatorname{rank}(S)=k.
\]
Let \(B\in\mathbb{R}^{n\times r}\) and \(A\in\mathbb{R}^{r\times n}\) be drawn
from an absolutely continuous distribution. Then
\begin{equation}
    \operatorname{rank}(S+BA)
    =
    \min\{n,k+r\}
    \qquad
    \textnormal{almost surely}.
    \label{eq:generic_completion_rank}
\end{equation}
In particular, if \(k\ge n-r\), then
\begin{equation}
    \operatorname{rank}(S+BA)=n
    \qquad
    \textnormal{almost surely}.
\end{equation}
\end{lemma}

\begin{proof}
Let
\[
    m:=\min\{n,k+r\}.
\]
Since \(\operatorname{rank}(S)=k\), there exist invertible matrices
\(P,Q\in\mathbb{R}^{n\times n}\) such that
\begin{equation}
    PSQ
    =
    \begin{pmatrix}
        I_k & 0\\
        0 & 0
    \end{pmatrix}.
    \label{eq:rank_normal_form}
\end{equation}
Rank is invariant under multiplication by invertible matrices, hence
\begin{equation}
    \operatorname{rank}(S+BA)
    =
    \operatorname{rank}(PSQ+PBAQ).
\end{equation}
Define
\[
    \widetilde B:=PB,
    \qquad
    \widetilde A:=AQ.
\]
Because \(P\) and \(Q\) are invertible and \(A,B\) are drawn from absolutely
continuous distributions, \(\widetilde A,\widetilde B\) are also absolutely
continuous.

We show that some \(m\times m\) minor of
\[
    PSQ+\widetilde B\widetilde A
\]
is not the zero polynomial in the entries of \(\widetilde A,\widetilde B\).
Let
\[
    t:=m-k.
\]
Then \(0\le t\le r\). Choose a deterministic value of
\(\widetilde B,\widetilde A\) as follows:
\[
    \widetilde B_{\{k+1,\ldots,k+t\},\{1,\ldots,t\}}=I_t,
    \qquad
    \widetilde A_{\{1,\ldots,t\},\{k+1,\ldots,k+t\}}=I_t,
\]
and set all other entries of \(\widetilde B,\widetilde A\) equal to zero.
Then \(\widetilde B\widetilde A\) places an identity block on coordinates
\(k+1,\ldots,k+t\). Therefore
\[
    PSQ+\widetilde B\widetilde A
\]
contains an \(m\times m\) identity block, and hence has rank at least \(m\).

Thus, at least one \(m\times m\) determinant polynomial is not identically
zero. Since the zero set of a nonzero polynomial has Lebesgue measure zero,
the same minor is nonzero almost surely for absolutely continuous
\(\widetilde A,\widetilde B\). Hence
\begin{equation}
    \operatorname{rank}(S+BA)\ge m
    \qquad
    \textnormal{almost surely}.
    \label{eq:generic_completion_lower}
\end{equation}

On the other hand, by subadditivity of rank,
\begin{equation}
    \operatorname{rank}(S+BA)
    \le
    \operatorname{rank}(S)+\operatorname{rank}(BA)
    \le
    k+r.
\end{equation}
Also \(\operatorname{rank}(S+BA)\le n\). Therefore
\begin{equation}
    \operatorname{rank}(S+BA)\le \min\{n,k+r\}=m.
    \label{eq:generic_completion_upper}
\end{equation}
Combining \cref{eq:generic_completion_lower,eq:generic_completion_upper}
gives
\[
    \operatorname{rank}(S+BA)=m=\min\{n,k+r\}
\]
almost surely.
\end{proof}

\subsection{Structural rank and maximum matchings}

\begin{definition}[Structural rank]
\label{def:structural_rank}
For a support pattern \(\Omega\subseteq[n]\times[n]\), define the bipartite
graph
\[
    G_\Omega = ([n]_{\mathrm{row}},[n]_{\mathrm{col}},\Omega).
\]
The structural rank of \(\Omega\) is the largest rank achievable by any matrix
whose nonzero entries are restricted to the support \(\Omega\).
\end{definition}

\begin{lemma}[Structural rank equals maximum matching]
\label{lem:structural_rank}
Let \(X_\Omega\in\mathbb{R}^{n\times n}\) be a sparse matrix supported on
\(\Omega\). Assume that its nonzero entries are algebraically generic. Then
\begin{equation}
    \operatorname{rank}(X_\Omega)=\nu(\Omega)
    \qquad
    \textnormal{almost surely}.
    \label{eq:structural_rank}
\end{equation}
\end{lemma}

\begin{proof}
Let
\[
    m:=\nu(\Omega).
\]
Since the graph \(G_\Omega\) has a matching of size \(m\), there exist row
and column sets \(I,J\subseteq[n]\), with \(|I|=|J|=m\), such that the
subgraph induced by \(I,J\) contains a perfect matching.

Consider the determinant of the submatrix \(X_{\Omega,I,J}\). Expanding this
determinant as a polynomial in the nonzero variables
\[
    \{x_{ij}:(i,j)\in\Omega\},
\]
the perfect matching contributes a monomial of the form
\begin{equation}
    \prod_{(i,j)\in M} x_{ij}.
\end{equation}
Different perfect matchings correspond to different monomials. Therefore this
determinant polynomial is not identically zero. Since the entries are
algebraically generic, the determinant is nonzero almost surely. Hence
\begin{equation}
    \operatorname{rank}(X_\Omega)\ge m.
    \label{eq:structural_rank_lower}
\end{equation}

Conversely, suppose that \(\operatorname{rank}(X_\Omega)\ge m+1\). Then there
exists an \((m+1)\times(m+1)\) minor that is not identically zero as a
polynomial in the nonzero entries. In the determinant expansion of this minor,
at least one permutation monomial must be present. Such a monomial corresponds
to a matching of size \(m+1\) in \(G_\Omega\), contradicting the maximality of
\(\nu(\Omega)=m\). Therefore
\begin{equation}
    \operatorname{rank}(X_\Omega)\le m.
    \label{eq:structural_rank_upper}
\end{equation}
Combining \cref{eq:structural_rank_lower,eq:structural_rank_upper} gives
\[
    \operatorname{rank}(X_\Omega)=\nu(\Omega)
\]
almost surely.
\end{proof}

\subsection{Hall deficiency and weighted Hall failure}

\begin{definition}[Neighborhood of a row set]
\label{def:neighborhood_row_set}
For a set \(I\subseteq[n]\) of row vertices, define its neighborhood in the
support graph \(G_\Omega\) by
\begin{equation}
    N_\Omega(I)
    :=
    \{j\in[n]: \exists i\in I \textnormal{ such that } (i,j)\in\Omega\}.
    \label{eq:neighborhood_definition}
\end{equation}
\end{definition}

\begin{lemma}[Deficiency form of Hall's theorem]
\label{lem:hall_deficiency}
For the bipartite graph \(G_\Omega\),
\begin{equation}
    n-\nu(\Omega)
    =
    \max_{I\subseteq[n]}
    \left(
    |I|-|N_\Omega(I)|
    \right).
    \label{eq:hall_deficiency}
\end{equation}
Consequently,
\[
    \nu(\Omega)<n-r
\]
if and only if there exists \(I\subseteq[n]\) such that
\begin{equation}
    |N_\Omega(I)|\le |I|-r-1.
    \label{eq:hall_bad_set}
\end{equation}
\end{lemma}

\begin{proof}
This is the standard deficiency form of Hall's theorem. A matching of size at
least \(n-r\) exists if and only if every row set \(I\subseteq[n]\) has
deficiency at most \(r\), namely
\[
    |I|-|N_\Omega(I)|\le r.
\]
Equivalently,
\[
    |N_\Omega(I)|\ge |I|-r
    \qquad
    \textnormal{for all } I\subseteq[n].
\]
Thus, the condition fails if and only if there exists \(I\subseteq[n]\) such
that
\[
    |N_\Omega(I)|\le |I|-r-1.
\]
\end{proof}

\begin{lemma}[Weighted Hall failure bound]
\label{lem:weighted_hall_bound}
Assume that the edge indicators \(\Omega_{ij}\) are independent Bernoulli
random variables with probabilities \(q_{ij}\). Let \(R:=r+1\). For
\(I,U\subseteq[n]\), define
\[
    \Lambda(I,U):=\sum_{i\in I,\ j\in U}q_{ij}.
\]
Then
\begin{equation}
    \mathbb{P}[\nu(\Omega)<n-r]
    \le
    \sum_{\substack{I,U\subseteq[n]\\ |I|\ge R,\ |U|\ge R,\ |I|+|U|\ge n+R}}
    \exp(-\Lambda(I,U)).
    \label{eq:weighted_hall_failure_bound}
\end{equation}
\end{lemma}

\begin{proof}
By \cref{lem:hall_deficiency}, the event \(\nu(\Omega)<n-r\) occurs if and
only if there exists \(I\subseteq[n]\) such that
\[
    |N_\Omega(I)|\le |I|-r-1.
\]
Let \(J:=N_\Omega(I)\) and \(U:=J^c\). Then no edges exist from \(I\) to \(U\).
Moreover,
\[
    |U|
    =
    n-|J|
    \ge
    n-|I|+r+1.
\]
Equivalently,
\[
    |I|+|U|\ge n+r+1=n+R.
\]
Also, such a violation can occur only when \(|I|\ge R\), and the previous
display implies \(|U|\ge R\). Thus, Hall failure implies the existence of sets
\(I,U\subseteq[n]\) satisfying
\[
    |I|\ge R,
    \qquad
    |U|\ge R,
    \qquad
    |I|+|U|\ge n+R,
\]
with no edges in the rectangle \(I\times U\).

For fixed \(I,U\), the probability that there are no edges in \(I\times U\)
is
\[
    \prod_{i\in I,\ j\in U}(1-q_{ij})
    \le
    \exp\left(-\sum_{i\in I,\ j\in U}q_{ij}\right)
    =
    \exp(-\Lambda(I,U)).
\]
Taking a union bound over all admissible \(I,U\) gives the result.
\end{proof}

\subsection{Proof of the deterministic rank proposition}

\begin{proof}[Proof of \cref{prop:matching_plus_lowrank}]
By \cref{lem:structural_rank}, the sparse matrix \(S_\Omega\) satisfies
\[
    \operatorname{rank}(S_\Omega)=\nu(\Omega)
\]
almost surely under the generic-values assumption. Applying
\cref{lem:generic_lowrank_completion} with \(S=S_\Omega\), we get
\begin{align}
    \operatorname{rank}(S_\Omega+BA)
    &=
    \min\{n,\operatorname{rank}(S_\Omega)+r\}
    \notag\\
    &=
    \min\{n,\nu(\Omega)+r\}
\end{align}
almost surely. In particular, if
\[
    \nu(\Omega)\ge n-r,
\]
then
\[
    \operatorname{rank}(S_\Omega+BA)=n
\]
almost surely.
\end{proof}

\subsection{Proof of the main full-rank theorem}

\begin{proof}[Proof of \cref{thm:angular_full_rank}]
Condition on \(Q,K\). Then the angular collision probabilities \(q_{ij}\) are
fixed numbers in \([0,1]\). By \cref{lem:weighted_hall_bound},
\[
    \mathbb{P}[\nu(\Omega)<n-r\mid Q,K]
    \le
    \Delta_r(Q,K).
\]
Therefore, with probability at least \(1-\Delta_r(Q,K)\), we have
\[
    \nu(\Omega)\ge n-r.
\]
On this event, \cref{prop:matching_plus_lowrank} implies
\[
    \operatorname{rank}(S_\Omega+BA)=n
\]
almost surely over the generic low-rank factors \(A,B\). Hence
\[
    \mathbb{P}\!\left[
    \operatorname{rank}(S_\Omega+BA)=n
    \,\middle|\, Q,K
    \right]
    \ge
    1-\Delta_r(Q,K).
\]
\end{proof}

\section{Interpretable High-Probability Corollaries}
\label{app:rank_corollaries}

The main theorem gives the full-rank probability through the quantity
\(\Delta_r(Q,K)\). In this section, we give several sufficient conditions under
which \(\Delta_r(Q,K)\le n^{-c}\), so that
\[
    \mathbb{P}\!\left[
    \operatorname{rank}(S_\Omega+BA)=n
    \,\middle|\, Q,K
    \right]
    \ge
    1-n^{-c}.
\]
These corollaries are not meant to exhaust all possible conditions. Rather,
they show different ways to interpret the rank guarantee: through uniform
cut-density, through a conservative worst-case angular lower bound, and through
an idealized isotropic mean model.

\subsection{A combinatorial rectangle bound}

\begin{lemma}[A binomial rectangle bound]
\label{lem:binomial_rectangle_bound}
Let \(R\in[n]\). Suppose \(a,h\in\{R,\ldots,n\}\) satisfy
\[
    a+h\ge n+R.
\]
Then
\begin{equation}
    \log\binom{n}{a}
    +
    \log\binom{n}{h}
    \le
    8\,\frac{ah}{n}\log\frac{en}{R}.
    \label{eq:binomial_rectangle_bound}
\end{equation}
\end{lemma}

\begin{proof}
We use the standard bound
\[
    \binom{n}{t}
    =
    \binom{n}{n-t}
    \le
    \left(\frac{en}{\min\{t,n-t\}}\right)^{\min\{t,n-t\}},
\]
with the convention that the corresponding term is zero when
\(\min\{t,n-t\}=0\).

First suppose \(a\le n/2\). Since \(a+h\ge n+R\), we have
\[
    n-h\le a-R\le a.
\]
Therefore,
\begin{align}
    \log\binom{n}{a}
    +
    \log\binom{n}{h}
    &=
    \log\binom{n}{a}
    +
    \log\binom{n}{n-h}
    \notag\\
    &\le
    a\log\frac{en}{a}
    +
    (n-h)\log\frac{en}{n-h}
    \notag\\
    &\le
    2a\log\frac{en}{a}
    \notag\\
    &\le
    2a\log\frac{en}{R}.
\end{align}
Also \(h\ge n/2\), so \(ah/n\ge a/2\). Hence
\[
    \log\binom{n}{a}
    +
    \log\binom{n}{h}
    \le
    4\,\frac{ah}{n}\log\frac{en}{R}.
\]
The case \(h\le n/2\) is symmetric.

It remains to consider the case \(a>n/2\) and \(h>n/2\). Then
\[
    \binom{n}{a}\binom{n}{h}\le 2^{2n},
\]
so
\[
    \log\binom{n}{a}
    +
    \log\binom{n}{h}
    \le
    2n\log 2.
\]
Since \(a>n/2\) and \(h>n/2\), we have \(ah/n>n/4\). Also
\(\log(en/R)\ge 1\). Therefore
\[
    8\,\frac{ah}{n}\log\frac{en}{R}
    \ge
    2n
    \ge
    2n\log 2.
\]
Combining the cases proves the claim.
\end{proof}

\subsection{Deficiency-aware collision density}

\begin{corollary}[Deficiency-aware collision density]
\label{cor:deficiency_aware_density}
Fix \(c>0\), and suppose \(0\le r<n\). Let
\[
    R:=r+1,
\]
and define
\begin{equation}
    p_{r,c}
    :=
    \frac{1}{n}
    \left(
    8\log\frac{en}{R}
    +
    \frac{(c+2)\log n}{R}
    \right).
    \label{eq:deficiency_aware_threshold}
\end{equation}
Suppose that for every pair of sets \(I,U\subseteq[n]\) satisfying
\[
    |I|\ge R,
    \qquad
    |U|\ge R,
    \qquad
    |I|+|U|\ge n+R,
\]
we have the cut-density lower bound
\begin{equation}
    \Lambda(I,U)
    =
    \sum_{i\in I,\ j\in U}q_{ij}
    \ge
    p_{r,c}|I||U|.
    \label{eq:deficiency_cut_density}
\end{equation}
Then, under the assumptions of \cref{thm:angular_full_rank},
\begin{equation}
    \mathbb{P}\!\left[
    \operatorname{rank}(S_\Omega+BA)=n
    \,\middle|\, Q,K
    \right]
    \ge
    1-n^{-c}.
    \label{eq:deficiency_density_rank_bound}
\end{equation}
\end{corollary}

\begin{proof}
Let
\[
    R:=r+1.
\]
Recall from \cref{thm:angular_full_rank} that the failure probability is
bounded by
\[
    \Delta_r(Q,K)
    =
    \sum_{\substack{I,U\subseteq[n]\\ |I|\ge R,\ |U|\ge R,\ |I|+|U|\ge n+R}}
    \exp(-\Lambda(I,U)).
\]
For fixed cardinalities
\[
    a:=|I|,
    \qquad
    h:=|U|,
\]
the cut-density assumption gives
\[
    \Lambda(I,U)
    \ge
    p_{r,c}ah.
\]
Therefore,
\begin{align}
    \Delta_r(Q,K)
    &\le
    \sum_{\substack{a,h\in\{R,\ldots,n\}\\ a+h\ge n+R}}
    \binom{n}{a}\binom{n}{h}
    \exp(-p_{r,c}ah).
    \label{eq:delta_grouped_by_rectangles}
\end{align}
By the definition of \(p_{r,c}\),
\[
    p_{r,c}ah
    =
    \frac{ah}{n}
    \left(
    8\log\frac{en}{R}
    +
    \frac{(c+2)\log n}{R}
    \right).
\]
Using \cref{lem:binomial_rectangle_bound},
\[
    \log\binom{n}{a}
    +
    \log\binom{n}{h}
    \le
    8\,\frac{ah}{n}\log\frac{en}{R}.
\]
Hence, for each admissible cardinality pair \((a,h)\),
\begin{align}
    \binom{n}{a}\binom{n}{h}
    \exp(-p_{r,c}ah)
    &\le
    \exp\left(
    -
    \frac{ah}{nR}(c+2)\log n
    \right).
\end{align}
Since \(a,h\ge R\) and \(a+h\ge n+R\), the product \(ah\) is minimized at
\((a,h)=(R,n)\) or \((n,R)\). Thus
\[
    ah\ge nR.
\]
Therefore each admissible cardinality pair contributes at most
\[
    n^{-(c+2)}.
\]
There are at most \(n^2\) admissible pairs \((a,h)\). Consequently,
\[
    \Delta_r(Q,K)
    \le
    n^2\cdot n^{-(c+2)}
    =
    n^{-c}.
\]
Applying \cref{thm:angular_full_rank} completes the proof.
\end{proof}

The threshold \(p_{r,c}\) is deficiency-aware. When \(r=0\), the sparse graph
itself must contain a perfect matching, and the sufficient sparse density has
the familiar \(O(\log n/n)\) scaling. When \(r>0\), the sparse graph is allowed
to have matching deficiency up to \(r\), and the rank-\(r\) low-rank branch
fills the missing directions.

\subsection{Conservative worst-case angular condition}

\begin{corollary}[Conservative worst-case angular condition]
\label{cor:worst_case_angular_condition}
Fix \(c>0\), suppose \(0\le r<n\), and let \(p_{r,c}\) be defined as in
\cref{eq:deficiency_aware_threshold}. Let
\[
    \rho_{\min}:=\min_{i,j}\rho_{ij}.
\]
If
\begin{equation}
    1-\exp(-L_{\mathrm{s}}\rho_{\min}^{\gamma})
    \ge
    p_{r,c},
    \label{eq:worst_case_angular_condition}
\end{equation}
then, under the assumptions of \cref{thm:angular_full_rank},
\begin{equation}
    \mathbb{P}\!\left[
    \operatorname{rank}(S_\Omega+BA)=n
    \,\middle|\, Q,K
    \right]
    \ge
    1-n^{-c}.
\end{equation}
Equivalently, when \(p_{r,c}<1\) and \(\rho_{\min}>0\), it is sufficient that
\begin{equation}
    L_{\mathrm{s}}
    \ge
    \frac{-\log(1-p_{r,c})}{\rho_{\min}^{\gamma}}.
    \label{eq:sufficient_number_sparse_trials}
\end{equation}
\end{corollary}

\begin{proof}
 For every pair \((i,j)\),
\[
    q_{ij}
    \ge
    1-\exp(-L_{\mathrm{s}}\rho_{ij}^{\gamma}).
\]
By definition of \(\rho_{\min}\),
\[
    \rho_{ij}\ge \rho_{\min}
    \qquad
    \textnormal{for all } i,j.
\]
Therefore,
\[
    q_{ij}
    \ge
    1-\exp(-L_{\mathrm{s}}\rho_{\min}^{\gamma}).
\]
If
\[
    1-\exp(-L_{\mathrm{s}}\rho_{\min}^{\gamma})
    \ge
    p_{r,c},
\]
then \(q_{ij}\ge p_{r,c}\) for all \(i,j\). Hence, for every admissible
rectangle \(I\times U\),
\[
    \Lambda(I,U)
    =
    \sum_{i\in I,\ j\in U}q_{ij}
    \ge
    p_{r,c}|I||U|.
\]
Thus the cut-density condition in \cref{cor:deficiency_aware_density} holds,
and the desired high-probability full-rank conclusion follows.

Finally, when \(p_{r,c}<1\) and \(\rho_{\min}>0\), the inequality
\[
    1-\exp(-L_{\mathrm{s}}\rho_{\min}^{\gamma})
    \ge
    p_{r,c}
\]
is equivalent to
\[
    \exp(-L_{\mathrm{s}}\rho_{\min}^{\gamma})
    \le
    1-p_{r,c},
\]
which holds whenever
\[
    L_{\mathrm{s}}
    \ge
    \frac{-\log(1-p_{r,c})}{\rho_{\min}^{\gamma}}.
\]
\end{proof}

\begin{remark}[Worst-case angular conditions are pessimistic]
\label{rem:worst_case_angular_pessimistic}
\Cref{cor:worst_case_angular_condition} uses the minimum angular similarity
over all \(n^2\) query-key pairs and is therefore conservative. In practice,
full rank does not require every pair to have large collision probability; it
only requires that no Hall-relevant rectangle has too little total collision
mass. Thus, some individual pairs may have very small collision probability,
as long as no large query set is separated from too many key vertices.
\end{remark}

\subsection{Isotropic mean Hoeffding certificate}

\begin{corollary}[Isotropic mean Hoeffding certificate]
\label{cor:isotropic_mean_hoeffding_certificate}
Fix \(c>0\), \(\delta\in(0,1)\), and suppose \(0\le r<n\). Let
\[
    R:=r+1 .
\]
For \(a,h\in\{R,\ldots,n\}\) satisfying \(a+h\ge n+R\), define
\begin{equation}
    \tau_{a,h,c}
    :=
    \frac{
    \log\binom{n}{a}
    +
    \log\binom{n}{h}
    +
    (c+2)\log n
    }{ah},
    \label{eq:size_dependent_hall_threshold}
\end{equation}
and
\begin{equation}
    \beta_{a,h,\delta}
    :=
    \sqrt{
    \frac{
    \log\binom{n}{a}
    +
    \log\binom{n}{h}
    +
    \log(n^2/\delta)
    }{2ah}
    } .
    \label{eq:hoeffding_rectangle_deviation}
\end{equation}
Let
\begin{equation}
    \Theta_{r,c,\delta}
    :=
    \max_{\substack{a,h\in\{R,\ldots,n\}\\ a+h\ge n+R}}
    \left(
    \tau_{a,h,c}
    +
    \beta_{a,h,\delta}
    \right).
    \label{eq:isotropic_mean_threshold}
\end{equation}

Assume an idealized isotropic collision-probability model in which the
entries \(q_{ij}\in[0,1]\) are independent random variables with common mean
\(\mu_q\). Moreover, assume that the isotropic angular estimate is a
conservative lower approximation to this mean:
\begin{equation}
    \mu_q
    \ge
    \bar q_{\mathrm{iso}}-\varepsilon_{\mathrm{iso}},
    \qquad
    \bar q_{\mathrm{iso}}
    :=
    1-\left(1-2^{-\gamma}\right)^{L_{\mathrm{s}}},
    \label{eq:isotropic_collision_mean}
\end{equation}
where \(\varepsilon_{\mathrm{iso}}\ge0\) measures the approximation slack.
In the sparse-collision regime \(L_{\mathrm{s}}2^{-\gamma}\ll1\),
\[
    \bar q_{\mathrm{iso}}
    =
    L_{\mathrm{s}}2^{-\gamma}
    +
    O\!\left(L_{\mathrm{s}}^2 2^{-2\gamma}\right).
\]

If
\begin{equation}
    \bar q_{\mathrm{iso}}-\varepsilon_{\mathrm{iso}}
    \ge
    \Theta_{r,c,\delta},
    \label{eq:isotropic_mean_condition}
\end{equation}
then, with probability at least \(1-\delta\) over the draw of the probability
matrix \((q_{ij})\), all Hall-relevant rectangles satisfy
\begin{equation}
    \frac{1}{|I||U|}
    \sum_{i\in I,\ j\in U} q_{ij}
    \ge
    \tau_{|I|,|U|,c}.
    \label{eq:all_rectangle_means_large}
\end{equation}
Consequently, under \cref{as:independent_collision_model,as:generic_values_lowrank},
with probability at least \(1-\delta\) over the draw of \((q_{ij})\),
\begin{equation}
    \mathbb{P}\!\left[
    \operatorname{rank}(S_\Omega+BA)=n
    \,\middle|\, (q_{ij})_{i,j}
    \right]
    \ge
    1-n^{-c}.
    \label{eq:conditional_full_rank_from_isotropic_mean}
\end{equation}
Equivalently, over the joint randomness of the collision probabilities,
the sparse support, and the low-rank factors,
\begin{equation}
    \mathbb{P}\!\left[
    \operatorname{rank}(S_\Omega+BA)=n
    \right]
    \ge
    1-\delta-n^{-c}.
    \label{eq:joint_full_rank_from_isotropic_mean}
\end{equation}
\end{corollary}

\begin{proof}
Fix admissible cardinalities \(a,h\in\{R,\ldots,n\}\) satisfying
\[
    a+h\ge n+R.
\]
For fixed sets \(I,U\subseteq[n]\) with \(|I|=a\) and \(|U|=h\), define
\[
    \overline q(I,U)
    :=
    \frac{1}{ah}
    \sum_{i\in I,\ j\in U}q_{ij}.
\]
Under the idealized isotropic collision-probability model, the entries
\(q_{ij}\in[0,1]\) are independent with common mean \(\mu_q\). Hence Hoeffding's
inequality gives
\[
    \mathbb{P}\left[
    \overline q(I,U)
    <
    \mu_q-\beta_{a,h,\delta}
    \right]
    \le
    \exp(-2ah\beta_{a,h,\delta}^2).
\]
By the definition of \(\beta_{a,h,\delta}\),
\[
    \exp(-2ah\beta_{a,h,\delta}^2)
    =
    \exp\left(
    -\log\binom{n}{a}
    -\log\binom{n}{h}
    -\log(n^2/\delta)
    \right).
\]
Taking a union bound over all \(\binom{n}{a}\binom{n}{h}\) pairs of sets with
these cardinalities, the probability that any such rectangle violates the
bound is at most
\[
    \frac{\delta}{n^2}.
\]
Taking another union bound over at most \(n^2\) admissible cardinality pairs
\((a,h)\), we obtain that, with probability at least \(1-\delta\), every
Hall-relevant rectangle satisfies
\[
    \overline q(I,U)
    \ge
    \mu_q-\beta_{|I|,|U|,\delta}.
\]

By assumption,
\[
    \mu_q
    \ge
    \bar q_{\mathrm{iso}}-\varepsilon_{\mathrm{iso}}
    \ge
    \Theta_{r,c,\delta}.
\]
Since
\[
    \Theta_{r,c,\delta}
    =
    \max_{\substack{a,h\in\{R,\ldots,n\}\\ a+h\ge n+R}}
    \left(
    \tau_{a,h,c}
    +
    \beta_{a,h,\delta}
    \right),
\]
it follows that every Hall-relevant rectangle satisfies
\[
    \overline q(I,U)
    \ge
    \tau_{|I|,|U|,c}.
\]
Equivalently,
\[
    \Lambda(I,U)
    =
    \sum_{i\in I,\ j\in U}q_{ij}
    \ge
    |I||U|\tau_{|I|,|U|,c}.
\]
By the definition of \(\tau_{a,h,c}\),
\[
    \Lambda(I,U)
    \ge
    \log\binom{n}{|I|}
    +
    \log\binom{n}{|U|}
    +
    (c+2)\log n.
\]
Therefore,
\begin{align}
    \Delta_r
    &=
    \sum_{\substack{I,U\subseteq[n]\\ |I|\ge R,\ |U|\ge R,\ |I|+|U|\ge n+R}}
    \exp(-\Lambda(I,U))
    \notag\\
    &\le
    \sum_{\substack{a,h\in\{R,\ldots,n\}\\ a+h\ge n+R}}
    \binom{n}{a}\binom{n}{h}
    \exp\left(
    -\log\binom{n}{a}
    -\log\binom{n}{h}
    -(c+2)\log n
    \right)
    \notag\\
    &\le
    n^2\cdot n^{-(c+2)}
    =
    n^{-c}.
\end{align}
Applying \cref{thm:angular_full_rank} conditionally on the realized probability
matrix \((q_{ij})\) gives
\[
    \mathbb{P}\!\left[
    \operatorname{rank}(S_\Omega+BA)=n
    \,\middle|\, (q_{ij})_{i,j}
    \right]
    \ge
    1-n^{-c}
\]
with probability at least \(1-\delta\) over the draw of \((q_{ij})\).

Finally, by the union bound over the two sources of failure, namely failure of
the probability matrix to satisfy the rectangle-average condition and failure
of the sampled sparse graph to have matching size at least \(n-r\), the joint
success probability is at least
\[
    1-\delta-n^{-c}.
\]
\end{proof}

The role of \(\delta\) in \cref{cor:isotropic_mean_hoeffding_certificate} is
to control the probability that the random collision-probability matrix itself
has a low-density Hall-relevant cut. The term \(n^{-c}\) controls the
subsequent failure probability of the sampled sparse graph, conditioned on
that probability matrix. In the common sparse-collision regime, the practical
design rule suggested by \cref{eq:isotropic_mean_condition} is
\[
    L_{\mathrm{s}}2^{-\gamma}
    \gtrsim
    \Theta_{r,c,\delta}.
\]
Thus, increasing \(L_{\mathrm{s}}\) increases sparse coverage, while increasing
\(\gamma\) makes collisions more selective.

\subsection{Interpretation and technical remarks}

\begin{remark}[Interpretation]
\label{rem:rank_interpretation}
The rank guarantee is an expressivity statement. It shows that the hybrid
sparse + low-rank construction avoids the rank collapse of purely low-rank
attention. The sparse collision graph contributes rank through its maximum
matching, while the low-rank branch supplies \(r\) additional dense directions.
Thus,
\[
    \textnormal{sparse matching size}+\textnormal{low-rank dimension}\ge n
\]
is sufficient for the hybrid attention matrix to be full rank.
\end{remark}

\begin{remark}[Scope of the probabilistic assumptions]
\label{rem:scope_probabilistic_assumptions}
The corollaries above should be read as interpretable regimes in which the
error term in \cref{thm:angular_full_rank} becomes small. Some of these regimes
make simplifying assumptions, such as independent edge sampling or an idealized
isotropic model for the collision probabilities. These assumptions are not
intended to fully model every practical hashing implementation. Rather, they
expose the mechanism behind the algorithm: if angular sparse sampling produces
enough coverage across Hall-relevant cuts, then the sparse component has
matching deficiency at most \(r\), and the rank-\(r\) low-rank component fills
the remaining directions.
\end{remark}

\begin{remark}[Independence]
\label{rem:idealized_independence}
\Cref{thm:angular_full_rank} is stated under an independent angular edge
model. Practical LSH collisions may be dependent because the same hash
functions are reused across multiple query-key pairs. The deterministic
implication in \cref{prop:matching_plus_lowrank}, however, does not require
independence. It applies to any realized support \(\Omega\) satisfying
\(\nu(\Omega)\ge n-r\). Independence is used only to control the probability
that this matching condition fails.
\end{remark}

\begin{remark}[Deterministic version]
\label{rem:deterministic_version_app}
The deterministic implication
\[
    \nu(\Omega)\ge n-r
    \quad\Longrightarrow\quad
    \operatorname{rank}(S_\Omega+BA)=n
\]
holds under the generic sparse-values and generic low-rank assumptions, and
does not require independence of the sparse support. Independence is used only
to upper bound the probability that \(\nu(\Omega)<n-r\).
\end{remark}

\begin{remark}[When exact sparse values are not generic]
\label{rem:non_generic_values_app}
If the sparse exact values \(S_\Omega=\Omega\odot M^\star\) are deterministic
and not assumed to be generic, one can use the deterministic condition
\[
    \operatorname{rank}(S_\Omega)\ge n-r.
\]
Then \cref{lem:generic_lowrank_completion} alone implies
\[
    \operatorname{rank}(S_\Omega+BA)=n
\]
almost surely over the low-rank factors \(A,B\). The matching condition
\(\nu(\Omega)\ge n-r\) is a support-level sufficient condition for this rank
condition under generic sparse values.
\end{remark}

\begin{remark}[Independent safety sparsifier]
\label{rem:safety_sparsifier_app}
If one wants a literal independent-edge guarantee while preserving a practical
LSH collision rule, one may augment the LSH support by an independent safety
sparsifier:
\[
    \Omega
    =
    \Omega_{\mathrm{LSH}}
    \cup
    \Omega_{\mathrm{safe}},
    \qquad
    \Omega_{\mathrm{safe},ij}
    \sim
    \mathrm{Bernoulli}\!\left(p_{r,c}\right).
\]
By \cref{cor:deficiency_aware_density}, the safety support alone is sufficient
to guarantee
\[
    \operatorname{rank}(S_\Omega+BA)=n
\]
with probability at least \(1-n^{-c}\), under the generic sparse-values and
generic low-rank assumptions. Its expected number of additional edges is
\[
    n^2p_{r,c}
    =
    n
    \left(
    8\log\frac{en}{r+1}
    +
    \frac{(c+2)\log n}{r+1}
    \right).
\]
\end{remark}
\section{Experiment Hyperparameters}
\label{app:experiment_hyperparameters}
We evaluate all attention variants using encoder-style, non-causal Transformer architectures. Unless otherwise stated, we train with cross-entropy loss and the AdamW optimizer. We use dropout \(0.1\), set \texttt{qkv\_bias=False}, and fix the random seed to 42. For the hybrid methods, the gate is a two-layer MLP with SiLU activation and two independent sigmoid outputs, one for the sparse branch and one for the RACE branch. We do not normalize the gates.

\begin{table*}[!t]
\centering
\small
\caption{Main dataset-level hyperparameters. Here \(N\) denotes the input sequence length after tokenization or patchification.}
\label{tab:experiment_hyperparams}
\setlength{\tabcolsep}{3.5pt}
\resizebox{\textwidth}{!}{
\begin{tabular}{lcccccccccc}
\toprule
\textbf{Dataset / Task}
& \(\mathbf{N}\)
& \textbf{Layers}
& \textbf{Heads}
& \(\mathbf{d}\)
& \textbf{MLP dim}
& \textbf{Batch}
& \textbf{Grad. accum.}
& \textbf{LR}
& \textbf{WD}
& \textbf{Epochs} \\
\midrule
IMDB
& 512
& 1
& 2
& 128
& \(512\)
& 32
& 1
& \(1\times 10^{-5}\)
& \(5\times 10^{-5}\)
& 150 \\

Fashion-MNIST
& 784
& 2
& 4
& 384
& \(1536\)
& 32
& 1--2
& \(6\times 10^{-4}\)
& 0.1
& 150 \\

Oxford-IIIT Pet
& 16,384
& 2
& 4
& 384
& \(1536\)
& 32
& 1--2
& \(6\times 10^{-4}\)
& 0.1
& 150 \\

Flowers-102
& 16,384
& 2
& 4
& 384
& \(1536\)
& 32
& 1--2
& \(6\times 10^{-4}\)
& 0.1
& 150 \\

Food-101
& 16,384
& 8
& 8
& 512
& \(2048\)
& 8
& 4
& \(3\times 10^{-4}\)
& 0.001
& 100 \\

ArXiv classification
& 32,000
& 4
& 4
& 256
& 1024
& 8
& 16
& \(3\times 10^{-4}\)
& 0.01
& 33 \\

Text Retrieval
& 64,000
& 4
& 4
& 256
& 1024
& 4
& 16
& \(3\times 10^{-4}\)
& 0.01
& 50 \\
\bottomrule
\end{tabular}
}
\end{table*}

For the image experiments, Fashion-MNIST uses \(28\times 28\) grayscale images with patch size 1, giving \(N=784\). Oxford-IIIT Pet and Flowers-102 use \(512\times512\) RGB images with patch size 4, giving \(128\times128=16{,}384\) image tokens. For Food-101, we use the long-image setting with \(512\times512\) inputs, patch size 4, and \(N=16{,}384\). The vision models use a learnable class token and learnable positional embeddings.

For ArXiv classification, we tokenize documents using a basic English tokenizer, keep documents with at least 1,000 raw tokens, and perform class-balanced streaming packing. Documents from the same class are concatenated until the target sequence length is reached. For the reported ArXiv classification setting, we use \(N=32{,}000\), a vocabulary limit of 50,000, and packed train/test examples produced from the class-balanced long-document subset. For Text Retrieval, we construct binary retrieval-pair examples from the packed ArXiv documents. Each input is formatted as
\[
    [\mathrm{CLS}] \; \text{doc}_a \; [\mathrm{SEP}] \; \text{doc}_b ,
\]
with label 1 if the two documents come from the same ArXiv class and label 0 otherwise. We use 4,000 training pairs and 1,000 test pairs at \(N=64{,}000\).

\begin{table}[!t]
\centering
\small
\caption{Attention-specific hyperparameters.}
\label{tab:attention_hyperparams}
\setlength{\tabcolsep}{4pt}
\begin{tabular}{lcc}
\toprule
\textbf{Component} & \textbf{Text settings} & \textbf{Vision settings} \\
\midrule
RACE hash bits \(\gamma\)
& 4 for ArXiv/Retrieval, 3 for IMDB
& 2 \\

RACE tables \(L_s\)
& 4 for ArXiv/Retrieval, 2 for IMDB
& 5 \\

RACE ensembles \(M\)
& 1 for ArXiv/Retrieval, 2 for IMDB
& 1 \\

Sort\_LSH bits
& 5
& 4--5 \\

Sort\_LSH block size \(s\)
& 256 for ArXiv/Retrieval, 32 for IMDB
& 32 \\

Sort\_LSH min length
& 4096 for ArXiv/Retrieval, 256 for IMDB
& 256 \\

Neighbor blocks
& 0
& 0 \\

Gate hidden dim
& 64
& 128 \\

Gate normalization
& False
& False \\

\(\epsilon\) for denominator correction
& \(10^{-6}\)
& \(10^{-6}\) \\
\bottomrule
\end{tabular}
\end{table}

For ELSAA, the sparse branch is Sort\_LSH and the low-rank branch is RACE. The output is computed as
\[
    O
    =
    g_{\mathrm{sparse}}\,m_{\mathrm{sparse}}\,O_{\mathrm{sparse}}
    +
    g_{\mathrm{race}}\,O_{\mathrm{race}},
\]
where \(g_{\mathrm{sparse}}\) and \(g_{\mathrm{race}}\) are token-wise sigmoid gates. In the denominator-aware version, we use
\[
    m_{\mathrm{sparse}}
    =
    \frac{d_{\mathrm{sparse}}}
    {d_{\mathrm{sparse}}+\lambda d_{\mathrm{race}}+\epsilon}.
\]
For the scalar-\(\lambda\) version, \(\lambda\) is parameterized as
\[
    \lambda = \exp(\ell_\lambda),
\]
initialized with \(\lambda=1.0\), and learned during training. For the input-dependent version, we use
\[
    \lambda_i = c + \sigma(w^\top q_i + b),
\]
where \(c\) is initialized to 0.3, is learnable, and is constrained to be nonnegative. The bias is initialized so that the initial average target is approximately 0.8. We detach \(q_i\) from the lambda path, use \(\lambda_{\min}=10^{-6}\), and initialize \(w\) with standard deviation \(10^{-3}\).

For IMDB, we build a long-review subset at \(N=512\). Reviews with length between \(N\) and \(2N\) are kept directly, longer reviews are split into overlapping windows with stride \(N/2\), and shorter reviews from the same class are concatenated until they reach the target length. During training, we apply light EDA augmentation using random deletion with probability 0.05 or random token swaps.

For the vision datasets, we use standard data augmentation. Fashion-MNIST uses random horizontal flipping and random cropping with padding 4 during training. Oxford-IIIT Pet, Flowers-102, and Food-101 use random resized cropping and horizontal flipping for training, and resize followed by center crop for validation. RGB image datasets are normalized with ImageNet mean and standard deviation.
\section{Numerical Complexity Example}
\label{app:numerical_complexity_example}

\paragraph{Attention interaction count.}
We give a concrete example to illustrate the scale of the savings from the sparse--low-rank construction. For
\[
    N=32{,}000,\qquad s=256,\qquad L_{\mathrm{s}}=4,\qquad \gamma=4,
\]
full attention computes
\[
    N^2 = 1.024\times 10^9
\]
query--key interactions per head. The sparse sortLSH branch computes
\[
    Ns = 32{,}000\cdot 256 = 8.192\times 10^6
\]
selected interactions, which is \(0.8\%\) of full attention. The RACE branch uses
\[
    NL_{\mathrm{s}}2^\gamma
    =
    32{,}000\cdot 4\cdot 16
    =
    2.048\times 10^6
\]
bucket interactions, which is \(0.2\%\) of full attention. Therefore, ELSAA uses
\[
    N(s+L_{\mathrm{s}}2^\gamma)
    =
    32{,}000(256+64)
    =
    1.024\times 10^7
\]
interactions, which is \(1.0\%\) of full attention. This corresponds to roughly a \(99.0\%\) reduction in attention interactions compared with exact dense attention.

\paragraph{Hashing and sorting overhead.}
The hashing overhead is linear or near-linear in sequence length. The sortLSH branch requires angular hash projections and sorting, approximately
\[
    O(Nd\,\gamma_{\mathrm{sort}})+O(N\log N),
\]
where \(\gamma_{\mathrm{sort}}\) denotes the number of hash bits used by the sparse selector. The RACE branch requires soft-hash projections and bucket-feature construction, approximately
\[
    O(Nd\,L_{\mathrm{s}}\gamma)+O(NL_{\mathrm{s}}\gamma 2^\gamma).
\]
In our experiments, \(s\), \(L_{\mathrm{s}}\), \(\gamma\), and \(\gamma_{\mathrm{sort}}\) are fixed hyperparameters. Therefore, these terms scale linearly or near-linearly in \(N\), whereas dense attention scales quadratically. For sufficiently long sequences, the hashing and sorting overhead is dominated by the saved \(N^2\) query--key computation.

\paragraph{Parameter overhead.}
The additional learnable parameters introduced by ELSAA are small compared with the base Transformer layer. The standard attention projections
\[
    W_Q,W_K,W_V,W_O\in\mathbb{R}^{d\times d}
\]
already contribute \(O(d^2)\) parameters. In contrast, the scalar-\(\lambda\) version of ELSAA adds only \(O(1)\) parameters, a query-dependent \(\lambda\) version adds \(O(d)\) parameters, and the gate MLP adds \(O(dg)\) parameters for gate hidden width \(g\ll d\). Thus, the parameter overhead is negligible relative to the base attention projections, while the main savings come from reducing the input-dependent attention computation.

\end{document}